%% file: iclr2023_conference.tex
\documentclass{article}

\usepackage{arxiv, times}
\iclrfinalcopy

\usepackage[utf8]{inputenc} 
\usepackage[T1]{fontenc}    
\usepackage{hyperref}       
\usepackage{url}            
\usepackage{booktabs}       
\usepackage{amsfonts}       
\usepackage{nicefrac}       
\usepackage{microtype}      
\usepackage{xcolor}         

\input{math_commands.tex}

\usepackage{amsmath}
\usepackage{amssymb}
\usepackage{amsthm}
\usepackage{ulem} 
\usepackage{dsfont}
\usepackage{varwidth}
\usepackage{graphicx}
\usepackage{subfigure}

\newcommand{\loss}{\mathcal{L}}
\newcommand{\uniform}{\mathcal{U}}
\newcommand{\relu}{\texttt{ReLU}}
\newcommand{\bn}{\texttt{BN}}
\newcommand{\exponential}{\texttt{exponential}}
\newcommand{\bernoulli}{\texttt{Bernoulli}}
\newcommand{\normal}{\mathcal{N}}
\newcommand{\mmd}{\texttt{MMD}}

\newcommand{\kl}{\texttt{KL}}
\newcommand{\fcn}{\texttt{FCN}}

\newcommand{\diag}{\texttt{diag}}
\newcommand{\supp}{\texttt{supp}}
\newcommand{\subexp}{\texttt{SE}}

\newcommand{\vepsilon}{\boldsymbol{\epsilon}}
\newcommand{\vxi}{\boldsymbol{\xi}}

\newtheorem{theorem}{Theorem}
\newtheorem{claim}{Claim}
\newtheorem{definition}{Definition}
\newtheorem{lemma}{Lemma}
\newtheorem{proposition}{Proposition}

\newenvironment{talign}
 {\let\displaystyle\textstyle\align}
 {\endalign}
 
\newenvironment{talign*}
 {\let\displaystyle\textstyle\csname align*\endcsname}
 {\endalign}

\newtheorem{innercustomgeneric}{\customgenericname}
\providecommand{\customgenericname}{}
\newcommand{\newcustomtheorem}[2]{%
	\newenvironment{#1}[1]
	{%
		\renewcommand\customgenericname{#2}%
		\renewcommand\theinnercustomgeneric{##1}%
		\innercustomgeneric
	}
	{\endinnercustomgeneric}
}
\newcustomtheorem{customthm}{Theorem}

\makeatletter
\def\@fnsymbol#1{\ensuremath{\ifcase#1\or \dagger\or \ddagger\or
   \mathsection\or *\or \mathparagraph\or \|\or **\or \dagger\dagger
   \or \ddagger\ddagger \else\@ctrerr\fi}}
\makeatother

\title{Unsupervised Learning of Initialization in Deep Neural Networks via Maximum Mean Discrepancy}

\author{Cheolhyoung Lee\thanks{New York University}\\
\scriptsize{\texttt{cheolhyoung.lee@nyu.edu}} \\
\And
Kyunghyun Cho\footnotemark[1]~~\thanks{Prescient Design, Genentech}~~\thanks{CIFAR Fellow}\\
\scriptsize{\texttt{kyunghyun.cho@nyu.edu}}\\
}
\begin{document}

\maketitle

\begin{abstract}
Despite the recent success of stochastic gradient descent in deep learning, it is often difficult to train a deep neural network with an inappropriate choice of its initial parameters. Even if training is successful, it has been known that the initial parameter configuration may negatively impact generalization. In this paper, we propose an unsupervised algorithm to find good initialization for input data, given that a downstream task is $d$-way classification. We first notice that each parameter configuration in the parameter space corresponds to one particular downstream task of $d$-way classification. 
We then conjecture that the success of learning is directly related to how diverse downstream tasks are in the vicinity of the initial parameters.
We thus design an algorithm that encourages small perturbation to the initial parameter configuration leads to a diverse set of $d$-way classification tasks.
In other words, the proposed algorithm ensures a solution to {\it any} downstream task to be near the initial parameter configuration.
We empirically evaluate the proposed algorithm on various tasks derived from MNIST with a fully connected network.
In these experiments, we observe that our algorithm improves average test accuracy across most of these tasks, and that such improvement is greater when the number of labelled examples is small.
\end{abstract}

\section{Introduction}

Initialization of parameters has long been identified as playing a critical role in improving both the convergence and generalization performance of deep neural networks \citep{glorot2010understanding, erhan2010does, he2015delving}. In recent years, however, various normalization techniques, such as batch normalization \citep{ioffe2015batch}, layer normalization \citep{ba2016layer}, and weight normalization \citep{salimans2016weight}, have been found to somewhat reduce this heavy reliance on the initialization of parameters. 
The normalization techniques have done so by preserving some of the conditions that motivated various initialization schemes throughout training more explicitly. For instance, batch normalization normalizes each neuron to have zero mean and unit variance across examples within a minibatch, which is what Xavier initialization \citep{glorot2010understanding} and He initialization \citep{he2015delving} aim to achieve in an ideal situation.

Although batch normalization has been widely used for training deep neural networks \citep{he2016deep,tan2019efficientnet}, there are a small number of studies about why it helps training \citep{santurkar2018does}. Rather than revealing its theoretical effect, several researchers studied whether batch normalization is really necessary by training deep neural networks without batch normalization. \citet{zhang2019fixup} have proposed Fixup initialization replacing batch normalization in ResNet \citep{he2016deep} by adding additional parameters to each residual block. \citet{brock2021high} have also succeeded to train ResNet with adaptive gradient clipping that adjusts unit-wise ratio of gradient norms to parameter norms during training. The similarity among their algorithms and batch normalization is that they add their own schemes to adaptively supervise optimization of the deep neural networks.

We suspect that the necessity for such adaptation comes from some neighborhood properties of an initial parameter configuration. Training is an optimization process finding an optimal parameter configuration which well-approximates a particular task derived from input data in the parameter space. It means that each parameter configuration corresponds to each task but this is not necessarily one-to-one. We hypothesize that training encourages the current parameter configuration to converge to the nearest optimal parameter configuration from the initial one. If there is no optimal solution near the initial parameter configuration, then the current parameter configuration either deviates from the initial parameter configuration (exploding gradient) or stays around it (vanishing gradient). We thus propose an algorithm to find a initial parameter configuration that can solve various tasks in their neighborhood.

Before finding such initial parameter configuration, we first need to check whether a given network solves any task derived from the input data. \citet{zhang2016understanding} empirically showed that over-parametrization of deep neural networks enables them to memorize the entire dataset so that they can be fitted to its arbitrary target task. Based on this, \citet{pondenkandath2018leveraging} have empirically demonstrated that pre-training on random labels can accelerate training on downstream tasks. However, \citet{maennel2020neural} has shown that the random label pre-training sometimes hurts the convergence of fine-tuning on the downstream tasks. They also presented that the pre-trained model generalizes worse than randomly initialized networks even if the random label pre-training promotes learning on the downstream task. For these studies, we further conjecture that a given over-parametrized network can solve any task in its parameter space, but it cannot do this at a single parameter configuration.

We therefore decide to utilize a set of parameter configurations, where we can find an optimal parameter configuration for any target task. If this set can be accumulated to the vicinity of one parameter configuration, we view this configuration as a good initial parameter configuration. To do this, we first restrict possible downstream tasks to $d$-way classification to make the model output domain be the same as a $(d-1)$-dimensional unit simplex defined in \eqref{def:simplex}. We then define a neighbor of the initial parameter configuration as small perturbation to this. Our unsupervised algorithm encourages each neighbor to solve a different task so that optimizers based on stochastic gradient descent such as Adam \citep{kingma2014adam} can easily find a solution near our initial parameter configuration.  

We offer the mathematical statement for our conjecture in \S\ref{sec:uniform}, and propose an optimization problem to satisfy our claim for a given input. In doing so, we observe two possible degenerate cases to achieve our goal. In \S\ref{sec:under-class} and \S\ref{sec:input_ignoring}, we present how to avoid these unwanted situations. We validate our algorithm by various binary tasks derived from MNIST \citep{lecun1998gradient} 
in \S\ref{sec:main_exp}. From these experiments, we observe that fine-tuning deep neural networks from our initial parameters improves average test accuracy across the various binary tasks, and this gain is greater when the number of labelled examples is small.


\section{Preliminaries and notations}

\paragraph{Norms} 

Unless explicitly stated, a norm $\|\cdot\|$ refers to $\normltwo$ norm. We denote the Frobenius norm of a matrix $\mA\in\mathbb{R}^{m\times n}$ by 
$\|\mA\|_F=\sqrt{\sum_{i=1}^m\sum_{j=1}^n A_{ij}^2}$,
where $A_{ij}$ is the $(i,j)$-th entry of $\mA$. We write the $\normltwo$ operator norm of $\mA$ as
$\|\mA\|^*=\sup_{\|\vx\|=1} \|\mA\vx\|$,
where $\vx\in\R^n$.
\paragraph{Supports} For a distribution $p(\vx)$, we write its support as 
$\supp(p(\vx))=\{\vx\in\R^n\mid p(\vx)>0\}.$

\paragraph{Model prediction} 

A model prediction for $d$-way classification is a point in the $(d-1)$-dimensional unit simplex $\Delta^{d-1}\subset\R^d$ defined by
\begin{align}
\label{def:simplex}
    \Delta^{d-1}=\left\{(p_1,p_2,\cdots,p_d)\in\R_{{\geq}0}^{d} :\sum_{i=1}^{d} p_i = 1\right\},
\end{align}
where $\R_{\geq0}$ is the set of non-negative real numbers. We refer to a prediction of the model parametrized by $\vtheta$ for an input $\vx$, as $\vf_\vtheta(\vx)\in\Delta^{d-1}$.

\paragraph{Uniform distribution over $\Delta^{d-1}$} 
In this paper, we mainly deal with the uniform distribution over $\Delta^{d-1}$, $\uniform(\Delta^{d-1})$. We can generate its random sample $\vu$ by 
\begin{align}
\label{eq:sampling_uniform_simplex}
    \vu=\left(\frac{\ve_1}{\sum_{i=1}^d \ve_i},\frac{\ve_2}{\sum_{i=1}^d \ve_i},\cdots,\frac{\ve_d}{\sum_{i=1}^d \ve_i}\right),
\end{align}
where each $\ve_i$ is independently drawn from $\exponential(1)$ \citep{marsaglia1961uniform}.


\paragraph{Maximum mean discrepancy (MMD)} 

The MMD \citep{gretton2012kernel} is a framework for comparing two distributions $p(\vx)$ and $q(\vy)$ when we have samples from both distributions. The kernel MMD is defined by 
\begin{align}
\label{def:mmd}
    \mmd(p(\vx),q(\vy);\gamma) =&\E_{\vx\sim p(\vx)}\E_{\vx'\sim p(\vx)}[k_\gamma(\vx, \vx')]
    \\
    &-2\E_{\vx\sim p(\vx)}\E_{\vy\sim q(\vy)}[k_\gamma(\vx, \vy)]
    \nonumber
    \\
    &+\E_{\vy\sim q(\vy)}\E_{\vy'\sim q(\vy)}[k_\gamma(\vy, \vy')],
    \nonumber
\end{align}
where $k$ is a kernel function. A Gaussian kernel is often used, i.e., $k_\gamma(\vx,\vy)=\exp\left(-\frac{\|\vx-\vy\|^2}{2\gamma^2}\right)$. \citet{gretton2012kernel} showed that $p(\vx)=q(\vy)$ in distribution if and only if $\mmd(p(\vx),q(\vy);\gamma)=0$.


\section{Unsupervised learning of initialization}
\label{sec:theory}

We start by conjecturing that the parameter configuration for any $d$-way classification must be in the vicinity of {\it good} initial parameters. In other words, a parameter configuration, that solves any $d$-way classification task, is near the initial parameter configuration, so that such configuration can be readily found by stochastic gradient descent using labelled examples. The question we answer here is then how we can identify such an initial parameter configuration given a set of unlabelled examples. 

\subsection{Uniformity over all mappings}
\label{sec:uniform}


Let $\vf_\vtheta(\vx)\in \R^d$ be an output of a deep neural network parametrized by $\vtheta\in\R^m$ given an input $\vx\in \R^n$ sampled from an input distribution $p(\vx)$. In supervised learning, there is a target mapping $\vf^*$ 
defined on $\supp(p(\vx))$, and we want to find $\vtheta^*$ that 
\begin{align}
\label{eq:general_supervised_learning}
    \min_{\vtheta\in\R^m} l(\vf_\vtheta, \vf^*),
\end{align}
for a given loss function $l$. For example, we often use $l(\vf_\vtheta, \vf^*)=\E_{\vx\sim p(\vx)}[\|\vf_\vtheta(\vx)-\vf^*(\vx)\|^2]$ for regression and $l(\vf_\vtheta, \vf^*)=\E_{\vx\sim p(\vx)}[\kl(\vf^*(\vx)||\vf_\vtheta(\vx))]$ for classification task, where $\kl(\vf^*(\vx)||\vf_\vtheta(\vx))$ is the Kullback-Leibler (KL) divergence from $\vf_\vtheta(\vx)$ to $\vf^*(\vx)$. 

In deep learning, it is usual to search for an optimal solution $\vtheta^*$ from \eqref{eq:general_supervised_learning} in the full parameter space $\R^m$ by using a first-order optimizer, such as SGD and Adam \citep{kingma2014adam}. In this process, \citet{hoffer2017train} have demonstrated however that
\begin{align}
\label{eq:ultra_slow_diffusion}
    \|\vtheta_t - \vtheta_0\| \sim \log t,
\end{align}     
where $\vtheta_t$ is a vector of parameters at the $t$-th optimization step and $\vtheta_0$ is that of initial parameters. 
In other words, 
the rate of deviation from $\vtheta_0$ decreases as training progresses. It means that the first order optimizer tends to find an optimal solution near the initial point. We thus rewrite \eqref{eq:general_supervised_learning} as
\begin{talign}
\label{eq:nbd_supervised_learning}
    \vtheta^* = \argmin_{\vtheta\in \sB_r(\vtheta_0)} l(\vf_\vtheta, \vf^*),
\end{talign}
where $\sB_r(\vtheta_0)$ is a $r$-ball centered at $\vtheta_0$, $\sB_r(\vtheta_0)=\{\vtheta\in\R^d:\|\vtheta-\vtheta_0\|<r\}$.

With this in our mind, what is the good initialization $\vtheta_0$ for \eqref{eq:nbd_supervised_learning}? To answer this question, we look at what kind of classifiers we have within $\sB_r(\vtheta_0)$. If $\vx$ is an example randomly drawn from the input distribution $p(\vx)$, the set of all possible model outputs from $\vx$ in $\sB_r(\vtheta_0)$ is 
\[
    \sF(\vx;\vtheta_0)=\{\vf_\vtheta(\vx):\vtheta\in\sB_r(\vtheta_0)\}.
\]
We define the collection of all possible target mappings from the input space into the $(d-1)$-dimensional unit simplex $\Delta^{d-1}$ defined in \eqref{def:simplex} as 
\[
    \mathcal{F} = \{\vf^*\mid \vf^*:\supp(p(\vx))\rightarrow \Delta^{d-1} \subset \R^d\}.
\]
If $\vtheta_0$ is a good initial configuration, $\sF(\vx;\vtheta_0)$ has to be $\Delta^{d-1}$. Otherwise, our model cannot approximate $\vf^*\in\mathcal{F}$ such that $\vf^*(\vx)\in\Delta^{d-1}\setminus\sF(\vx;\vtheta_0)$ in $\sB_r(\vtheta_0)$.  

To approximate all target mappings in $\mathcal{F}$ by $\vf_\vtheta$ near $\vtheta_0$ for $\vx$, there must be $\vtheta\in\sB_r(\vtheta_0)$ satisfying $\vf_\vtheta(\vx)=\vs$ for arbitrary $\vs\in\Delta^{d-1}$.
In other words, if we randomly pick $\vtheta$ in $\sB_r(\vtheta_0)$, the probability density of $\vf_\vtheta(\vx)=\vs$ for any $\vs\in\Delta^{d-1}$ is positive and should be the same over $\Delta^{d-1}$ without prior knowledge of target mappings.

\begin{claim}
\label{claim:uniformity}
    Denote the distribution of $\vy=\vf_{\vtheta}(\vx)$ given $\vx \sim p(\vx)$ over $\vtheta\sim\uniform(\sB_r(\vtheta_0))$ as $q_{\vx}(\vy;\vtheta_0,r)$.\footnote{
        Although $\vx$ is given, $\vf_\vtheta(\vx)$ is random due to the randomness of $\vtheta$.
    } 
    Then, $\vtheta_0$ is a good initialization if and only if $\supp(q_{\vx}(\vy;\vtheta_0,r))=\Delta^{d-1}$ and $q_{\vx}(\vy;\vtheta_0,r)$ is equal to $\uniform(\Delta^{d-1})$ in distribution, because we do not know which $\vs\in\Delta^{d-1}$ is more likely.
\end{claim}

To obtain $\vtheta_0$ satisfying Claim~\ref{claim:uniformity}, we build an optimization problem that makes $q_{\vx}(\vy;\vtheta_0,r)$ converge to $\uniform(\Delta^{d-1})$ in distribution for a given $\vx\sim p(\vx)$. 
The first step toward this goal is to use 
the maximum mean discrepancy (MMD) \citep{gretton2012kernel} from \eqref{def:mmd}. We define an example specific loss as
\begin{align}
\label{eq:mmd_uniform_nbd}
    \loss_{\vx}^{uni}(\vtheta_0;r,\Delta^{d-1},\gamma)
    =
    \mmd(q_{\vx}(\vy;\vtheta_0,r),\uniform(\Delta^{d-1});\gamma).
\end{align}

According to \citet{gretton2012kernel}, \eqref{eq:mmd_uniform_nbd} is equal to $0$ if and only if $q_{\vx}(\vy;\vtheta_0,r)$ is equal to $\uniform(\Delta^{d-1})$ in distribution. 
We can therefore find $\vtheta_0$ that satisfies Claim~\ref{claim:uniformity}, by minimizing \eqref{eq:mmd_uniform_nbd} with respect to $\vtheta_0$.

The minimization of \eqref{eq:mmd_uniform_nbd} with respect to $\vtheta_0$ needs samples from both $\uniform(\Delta^{d-1})$ and $\uniform(\sB_r(\vtheta_0))$. In the case of $\uniform(\Delta^{d-1})$, 
we draw samples using \eqref{eq:sampling_uniform_simplex}. For $\uniform(\sB_r(\vtheta_0))$, we relax it to $\normal(\vtheta_0,\mSigma)$ where $\mSigma=\diag(\sigma_1^2,\sigma_2^2,\cdots,\sigma_m^2)$ for two reasons: i) 
this applies the same with uniform, since we can change the value range for each parameter separately; ii) the normal distribution allows us to use the reparametrization trick to compute $\nabla_{\vtheta_0}\loss_{\vx}^{uni}(\vtheta_0;r,\Delta^{d-1},\gamma)$ from \eqref{eq:mmd_uniform_nbd} \citep{kingma2013auto}. Furthermore, as shown in Theorem~\ref{thm:prob_out_of_nbd} below, a proper choice of the covariance matrix makes Gaussian perturbation have similar effect as uniform perturbation:
\begin{theorem}
\label{thm:prob_out_of_nbd}
    Let $\vtheta\sim\normal(\vtheta_0, \diag(\sigma_1^2,\sigma_2^2,\cdots,\sigma_m^2))$ and $\alpha_* =\max_{i=1,2,\cdots,m} \sigma_i^2$. If $r^2$ is greater than $m\alpha_*$, then we have
    \begin{align}\label{eq:prob_out_of_nbd}
        \sP\left(\|\vtheta-\vtheta_0\|\geq\ r \right) \leq \exp\left(-\frac{1}{8}\min\left\{\eta^2,m\eta\right\}\right),
    \end{align}
    where $\eta=\frac{r^2}{m\alpha_*}-1$ (proved in \S\ref{a_sec:thm1_proof}). 
\end{theorem}

Theorem~\ref{thm:prob_out_of_nbd} implies that if we add a Gaussian perturbation $\vepsilon\sim\normal(\vzero,\mSigma)$ to $\vtheta_0$, then the perturbed parameter configuration, $\vtheta=\vtheta_0+\vepsilon$, is enough closed to $\vtheta_0$ with a high probability, when $\alpha^*=\max_i\sigma_i^2$ is sufficiently small. In other words, although $\normal(\vtheta_0,\mSigma)$ is not exactly equivalent to $\uniform(\sB_r(\vtheta_0))$ in distribution, these two distributions play a similar role in the view of generating random parameter configurations near $\vtheta_0$. We therefore rewrite \eqref{eq:mmd_uniform_nbd} to 
enable reparametrization trick, as below:
\begin{align}
\label{eq:mmd_normal_nbd}
    \loss_{\vx}^{uni}(\vtheta_0;\mSigma,\Delta^{d-1},\gamma)=
    &\E_{\vepsilon\sim\normal(\vzero,\mSigma)}\E_{\vepsilon'\sim\normal(\vzero,\mSigma)}[k_\gamma(\vf_{\vtheta_0+\vepsilon}(\vx), \vf_{\vtheta_0+\vepsilon'}(\vx))]
    \\
    &-2\E_{\vepsilon\sim\normal(\vzero,\mSigma)}\E_{\vu\sim \uniform(\Delta^{d-1})}[k_\gamma(\vf_{\vtheta_0+\vepsilon}(\vx), \vu)]
    \nonumber
    \\
    &+\E_{\vu\sim \uniform(\Delta^{d-1})}\E_{\vu'\sim \uniform(\Delta^{d-1})}[k_\gamma(\vu, \vu')],
    \nonumber
\end{align}
where $q_{\vx}(\vy;\vtheta_0,\mSigma)$ is the distribution of $\vf_\vtheta(\vx)$ given $\vx$ with $\vtheta\sim\normal(\vtheta_0,\mSigma)$. 
In other words, we add Gaussian noise to each parameter and encourage prediction for $\vx$ based on such perturbed parameter configuration to be well spread out over $\Delta^{d-1}$. 
From now on, we use $\vtheta_0+\vepsilon$ to denote the perturbed parameter configuration, with $\vepsilon\sim\normal(\vzero,\mSigma)$, to be more explicit about our use of reparametrization trick. 

\Eqref{eq:mmd_normal_nbd} is an example specific loss, and minimizing this with respect to $\vtheta_0$ only guarantees the existence of $\vtheta^*$
near $\vtheta_0$
satisfying $\vf^*=\vf_{\vtheta^*}$ for a single $\vx$. Hence, we take the expectation of \eqref{eq:mmd_normal_nbd} over the input distribution $p(\vx)$:
\begin{align}
\label{eq:uniformity_loss}
    \loss^{uni}(\vtheta_0;\mSigma,\Delta^{d-1},\gamma,p(\vx))
    =
    \E_{\vx\sim p(\vx)}[\loss_\vx^{uni}(\vtheta_0;\mSigma,\Delta^{d-1},\gamma)].
\end{align}
We minimize this expected loss to find an initial parameter configuration $\vtheta_0^*$ that satisfies Claim~\ref{claim:uniformity} for the input data on average. 
When done so, we can find $\vf_\vtheta$ within the close proximity of $\vtheta_0^*$ that approximates any $d$-way target mapping $\vf^*$, given $p(\vx)$. 

\subsection{Degeneracies and remedies}
\label{sec:degenerate_remedy}

Let $\vx_1,\vx_2,\cdots,\vx_M$ be random samples drawn from $p(\vx)$, $\vepsilon_1, \vepsilon_2,\cdots, \vepsilon_N$ be random perturbations from $\normal(\vzero,\mSigma)$, and $\vc_1,\vc_2,\cdots, \vc_N$ from $\uniform(\Delta^{d-1})$. If $\vtheta_0^{1}$ satisfies
\begin{align}
\label{eq:degenerate_ex}
    \vc_j
    =
    \vf_{\vtheta_0^{1}+\vepsilon_j}(\vx_1)
    =
    \vf_{\vtheta_0^{1}+\vepsilon_j}(\vx_2)
    =
    \cdots
    =
    \vf_{\vtheta_0^{1}+\vepsilon_j}(\vx_M),
\end{align}
for each $j$, then $\loss_{\vx_i}^{uni}(\vtheta_0^{1};\mSigma,\Delta^{d-1},\gamma)=0$ for all $i$. Hence, $\vtheta_0^{1}$ is one of the optimal solutions for \eqref{eq:uniformity_loss}.
In the case of $\vtheta_0^{1}$, each perturbed model near $\vtheta_0^{1}$ is a constant function, to which we refer as {\it input-output detachment}.
Furthermore, 
each of these constant functions may output a {\it degenerate} categorical distribution whose support does not cover all $d$ classes, for which we refer to this phenomenon as {\it degenerate softmax}. We empirically demonstrate that both degeneracies indeed occur when we train a fully connected network by minimizing $\loss^{uni}$ in \S\ref{a_sec:degenerate_cases_exp}.  
In this section, we present two regularization terms, to be added to \eqref{eq:uniformity_loss}, to avoid these two unwanted cases, respectively.

\subsubsection{Degenerate softmax} 
\label{sec:under-class}

We first address the latter issue of degenerate softmax. Since we have specified that the task of our interest is $d$-way classification, we prefer models that can classify inputs into all $d$ classes in the neighborhood of $\vtheta_0^*$. We thus impose a condition that there exists at least one example categorized into each and every class. We first define a set of the points $\sA_i$ classified into the $i$-th class as
\begin{align}
\label{def:ith_part_simplex}
    \sA_i
    =
    \left\{\va=(a_1,a_2,\cdots,a_d)\in\Delta^{d-1}: 
    a_i\geq a_j\textrm{ for all } j=1,2,\cdots,d \right\}.
\end{align}

Given $\vepsilon\sim\normal(\vzero,\mSigma)$, the probability of \textit{`the model at $\vtheta_0^*+\vepsilon$ classifies $\vx$ into the $i$-th class'} is $\sP_{\vx\sim p(\vx)}(\vf_{\vtheta_0^*+\vepsilon}(\vx)\in\sA_i)$. This probability should be positive for all $i=1,2,\cdots,d$ to avoid degenerate softmax at $\vtheta_0^*$. To satisfy this, we use Theorem~\ref{thm:prob_each_class} which offers a lower bound of $\sP_{\vx\sim p(\vx)}(\vf_{\vtheta_0^*+\vepsilon}(\vx)\in\sA_i)$ using the distance from the $i$-th vertex $\vv^{(i)}$:

\begin{theorem}
\label{thm:prob_each_class}
    Let $\vv^{(i)}=\left(v^{(i)}_1,v^{(i)}_2,\cdots,v^{(i)}_d\right)\in\Delta^{d-1}$, where $v^{(i)}_i=1$ and $\sA_i$ be a subset of $\Delta^{d-1}$, as defined in \eqref{def:ith_part_simplex}. Then, 
    \begin{align}
    \label{eq:lower_prob_each_class}
        \sP_{\vx\sim p(\vx)}(\vf_{\vtheta_0^*+\vepsilon}(\vx)\in\sA_i) \geq 1-\sqrt{d}\E_{\vx\sim p(\vx)}[\|\vv^{(i)}-\vf_{\vtheta_0^*+\vepsilon}(\vx)\|],
    \end{align}
    for a given $\vepsilon\sim\normal(\vzero,\mSigma)$ (proved in \S\ref{a_sec:thm2_proof}).
\end{theorem}

According to \eqref{eq:lower_prob_each_class},  
$\E_{\vx\sim p(\vx)}[\|\vv^{(i)}-\vf_{\vtheta_0^*+\vepsilon}(\vx)\|]<\frac{1}{\sqrt{d}}$ implies $\sP_{\vx\sim p(\vx)}(\vf_{\vtheta_0^*+\vepsilon}(\vx)\in\sA_i)>0$ for each $i$, given $\vepsilon\sim\normal(\vzero,\mSigma)$. This means that we can avoid degenerate softmax by minimizing
\begin{align}
\label{eq:underclass_loss}
    \mathcal{L}^{sd}(\vtheta_0;\mSigma,d,p(\vx))= \E_{\vepsilon\sim\normal(\vzero,\mSigma)}\left[\max\left\{\max_{i=1,2,\cdots,d}\E_{\vx\sim p(\vx)}[\|\vv^{(i)}-\vf_{\vtheta_0^*+\vepsilon}(\vx)\|],\frac{1}{\sqrt{d}}\right\}-\frac{1}{\sqrt{d}}\right].
\end{align}
This minimization pulls the softmax output toward the furthest vertex for each $\vepsilon\sim\normal(\vzero,\mSigma)$, 
eventually avoiding the issue of degenerate softmax.

\subsubsection{Input-output detachment}
\label{sec:input_ignoring}

Here, let us go back to the first issue of input-output detachment we identified in \eqref{eq:degenerate_ex}. This issue happens when each perturbed model near $\vtheta_0^{1}$ is a constant function. In other words, the Jacobian of the model's output with respect to the input is zero, and in the case of multi-layered neural networks, the Jacobian of the model's output with respect to one of the intermediate layers is zero. This largely prevents learning from $\vtheta_0^{1}$, because $\vtheta_0^{1}$ is surrounded by the parameter configurations from which learning cannot happen. 
We thus design an additional loss that regularizes the Jacobian of model prediction with respect to its input and hidden neurons to prevent the input-output detachment.

In the rest of this section, we consider $\vf$ as the logits instead of the values after applying softmax, in order to avoid an issue of saturation caused by softmax \citep{varga2017gradient}. Let $\vx_l\in \R^{n_l}$, for $l \in \left\{ 0,1,\cdots,L \right\}$, be a vector of pre-activated neurons at the $(l+1)$-th layer parametrized by $\vtheta^{(l+1)}_0$, where $\vx_0\in\R^{n_0}$ and $\vx_L\in\R^{n_L}=\R^d$ are an input vector and its corresponding output vector, respectively. $\vf_{\vtheta_0^{(l:L)}}$ is the function from $\R^{n_l}$ to $\R^d$, parametrized by $\vtheta_0^{(l+1)},\vtheta_0^{(l+2)},\cdots,\vtheta_0^{(L)}$. Let us now consider the effect of perturbing the input to such a function:
\begin{align}
\label{eq:jacobian_perturbation}
    \vf_{\vtheta_0^{(l:L)}}(\vx_l+\vxi_l)
    \approx 
    \vf_{\vtheta_0^{(l:L)}}(\vx_l)+\mJ_{\vtheta_0^{(l:L)}}(\vx_l)\vxi_l,
\end{align}
where $\mJ_{\vtheta_0^{(l:L)}}(\vx_l)\in\R^{d\times n_l}$ is the Jacobian matrix of $\vf_{\vtheta_0^{(l:L)}}$ with respect to $\vx_l$. 

We then look at \eqref{eq:jacobian_perturbation} entry-wise:
\begin{align}
\label{eq:jacobian_perturbation_entry}
    f_{\vtheta_0^{(l:L)}}^{(i)}(\vx_l+\vxi_l)
    \approx 
    f_{\vtheta_0^{(l:L)}}^{(i)}(\vx_l)+J_{\vtheta_0^{(l:L)}}^{(i)}(\vx_l)\vxi_l,
\end{align} 
where $f_{\vtheta_0^{(l:L)}}^{(i)}$ is the $i$-th entry of $\vf_{\vtheta_0^{(l:L)}}$, and $J_{\vtheta_0^{(l:L)}}^{(i)}$ is the $i$-th row of $\mJ_{\vtheta_0^{(l:L)}}$ for $i=1,2,\cdots,d$. From \eqref{eq:jacobian_perturbation_entry}, we can see that the absolute difference between $f_{\vtheta_0^{(l:L)}}^{(i)}(\vx_l)$ and $f_{\vtheta_0^{(l:L)}}^{(i)}(\vx_l+\vxi_l)$ can be well approximated by the absolute value of the gradient-perturbation product:
\begin{align}
\label{eq:jacobian_abs_change}
    \left|f_{\vtheta_0^{(l:L)}}^{(i)}(\vx_l+\vxi_l)-f_{\vtheta_0^{(l:L)}}^{(i)}(\vx_l)\right|
    \approx 
    \left|J_{\vtheta_0^{(l:L)}}^{(i)}(\vx_l)\vxi_l\right|.
\end{align}
Assuming the perturbation's norm to be unit, we can bound this quantity by the operator norm of the $i$-th row of Jacobian:
\begin{align}
\label{eq:jacobian_operator}
    \sup_{\|\vxi_l\|=1}\left|f_{\vtheta_0^{(l:L)}}^{(i)}(\vx_l+\vxi_l)-f_{\vtheta_0^{(l:L)}}^{(i)}(\vx_l)\right|
    \approx 
    \sup_{\|\vxi_l\|_2=1}\left|J_{\vtheta_0^{(l:L)}}^{(i)}(\vx_l)\vxi_l\right|
    =
    \left\|J_{\vtheta_0^{(l:L)}}^{(i)}(\vx_l)\right\|^*.
\end{align}
Since $J_{\vtheta_0^{(l:L)}}^{(i)}(\vx_l)$ is a row vector, i.e., a matrix of rank 1, the Frobenius norm
$\|\cdot\|_F$ is equivalent to the operator norm $\|\cdot\|^*$. This allows us to rewrite \eqref{eq:jacobian_operator} as 
\begin{align}
\label{eq:jacobian_frobenius}
    \sup_{\|\vxi_l\|_2=1}\left|f_{\vtheta_0^{(l:L)}}^{(i)}(\vx_l+\vxi_l)-f_{\vtheta_0^{(l:L)}}^{(i)}(\vx_l)\right|
    \approx
    \left\|J_{\vtheta_0^{(l:L)}}^{(i)}(\vx_l)\right\|_F.
\end{align}

According to \Eqref{eq:jacobian_frobenius}, 
if $\|J_{\vtheta_0^{(l:L)}}^{(i)}(\vx_l)\|_F$ is positive, our initial model $\vf_{\vtheta_0}$ is sensitive to the change in $\vx_l$. That is, it is not a constant function.

Per the derivation above, in order to avoid the input-output detachment, 
we can for instance impose that, for all $i=1,2,\cdots,d$,
\begin{align}
\label{eq:const_jaco}
    c
    =
    \left\|J_{\vtheta_0^{(0:L)}}^{(i)}(\vx_0)\right\|_F
    =
    \left\|J_{\vtheta_0^{(1:L)}}^{(i)}(\vx_1)\right\|_F
    =
    \cdots
    =
    \left\|J_{\vtheta_0^{(L-1:L)}}^{(i)}(\vx_{L-1})\right\|_F,
\end{align} 
where $c>0$ is a constant. Here, we set $c=1$ which has an equivalent effect of setting the parameters using the so-called He initialization \citep{he2015delving}, as shown in the following theorem:
\begin{theorem}
\label{thm:he_init}
    Let $\vf_{\vtheta_0}$ be a fully connected network with ReLU \citep{nair2010rectified} non-linearity. We write the layerwise non-linear transformation from $\vx_l$ to $\vx_{l+1}$ for $l\neq 0$ as 
    \begin{align*}
        \vf_{\vtheta_0^{(l:l+1)}}(\vx_l)
        =
        \mW^{(l+1)}\relu(\vx_l)+\vb^{(l+1)},
    \end{align*}
    where $\mW^{(l+1)}\in\R^{n_{l+1}\times n_l}$ is the weight matrix and $\vb\in\R^{n_{l+1}}$ is the bias vector. 
    Assume that each element of $\vx_l$ has a symmetric distribution at $0$ and all elements of $\vx_l$ are mutually independent. If the $(i,j)$-th entry of $\mW^{(l+1)}$, $W_{ij}^{(l+1)}$, is a random sample from $\normal(0,\sigma_l^2)$ and $\vb^{(l+1)}$ is $\vzero$, then the following equality holds for all $k=1,2,\cdots, n_{l+1}$ when $\sigma_l=\sqrt{\frac{2}{n_l}}$ with sufficiently large $n_l$:
    \begin{align}
    \label{eq:he_jaco_condition}
        1\approx\left\|J_{\vtheta_0^{(l:l+1)}}^{(k)}(\vx_l)\right\|_F
        =
        \|\mW^{(l+1)} \mathds{1}(\vx_l > 0)\|_F, 
    \end{align}
    where $\mathds{1}(\vx_l >0)$ turns each positive entry in $\vx_l$ to $1$ and $0$ otherwise (proved in \S\ref{a_sec:thm3_proof}).
\end{theorem}
In order to prevent input-output detachment, we thus 
introduce an additional regularization term:
\begin{align}
\label{eq:ignore_loss}
    \loss^{iod}(\vtheta_0;p(\vx))
    =
    \E_{\vx\sim p(\vx)}\left[\frac{1}{d}\sum_{i=1}^d \left\{\max_{l\in\{0,1,\cdots,L-1\}}\left(1-\left\|J_{\vtheta_0^{(l:L)}}^{(i)}(\vx_l)\right\|_F\right)^2\right\}\right],
\end{align}
where $\vx_l$ is a vector of pre-activated neurons at the $l$-th layer and $\vx_0$ is an input vector. By minimizing \eqref{eq:ignore_loss} with respect to $\vtheta_0$, we prevent $\vtheta_0$ from being constant, and consequently all nearby models as well, which we demonstrate empirically in \S\ref{a_sec:nearby_jaco_exp}. 

\subsection{Hyperparameters and our recommendation}
\label{sec:final_loss}

We designed three loss functions to find a good initial parameter configuration $\vtheta_0^*$ for $d$-way classification, using only unlabelled examples; i) $\loss^{uni}(\vtheta_0;\mSigma,\Delta^{d-1},\gamma)$ in \S\ref{sec:uniform}; ii) $\loss^{sd}(\vtheta_0;\mSigma,d)$ in \S\ref{sec:under-class}; iii) $\loss^{iod}(\vtheta_0)$ in \S\ref{sec:input_ignoring}.
$\loss^{uni}(\vtheta_0;\mSigma,\Delta^{d-1},\gamma)$ makes our model predictions be evenly spread over $\Delta^{d-1}$ centered on $\vtheta_0$. $\loss^{sd}(\vtheta_0;\mSigma,d)$ encourages the neighborhood of $\vtheta_0$ to have solutions specialized for $d$-way classification by preventing {\it degenerate softmax}. $\loss^{iod}(\vtheta_0)$ avoids the issue of  {\it input-output detachment}. We additively combine all these to form the final loss function:
\begin{align}
\label{eq:final_loss}
    \loss(\vtheta_0;\mSigma,\Delta^{d-1},\gamma, p(\vx),\lambda,\xi)=&\loss^{uni}(\vtheta_0;\mSigma,\Delta^{d-1},\gamma,p(\vx)) 
    \\
    & + \lambda\loss^{sd}(\vtheta_0;\mSigma,d,p(\vx))
    \nonumber
    \\
    &+\xi\loss^{iod}(\vtheta_0;p(\vx)).
    \nonumber 
\end{align}
In \S\ref{a_sec:need_all_exp}, we empirically present that $\loss^{sd}$ and $\loss^{iod}$ indeed prevent the degenerate softmax and the input-output detachment, and all these three loss functions in \eqref{eq:final_loss} are necessary to find a good initial parameter configuration. 
In the rest of this section, we provide guidelines on how to choose some of the hyperparameters.

We select the bandwidth of MMD in $\loss^{uni}$, $\gamma$, based on the median heuristic \citep{smola1998learning}. It uses the median of all pairwise distances for the Gaussian kernel in \eqref{eq:mmd_normal_nbd}. This technique is commonly used in many unsupervised learning based on the Gaussian kernel \citet{garreau2017large} such as kernel CCA \citep{bach2002kernel} and kernel two-sample test \citep{gretton2012kernel}. For more detailed description of the median heuristic in our experiments, see \S\ref{a_sec:med_heuristic}.  

For $\mSigma=\diag(\sigma_1^2,\sigma_2^2,\cdots,\sigma_m^2$) of both $\loss^{uni}$ and $\loss^{sd}$, each $\sigma_i^2$ corresponding to $\theta_{0,i}$ is set based on the number of neurons connected to $\theta_{0,i}$. For instance, if $\theta_{0,i}$ is the entry of either $\mW\in\R^{n_{out}\times n_{in}}$ or $\vb\in\R^{n_{out}}$ (i.e., a parameter in a fully-connected layer), we set $\sigma_i$ to $\sqrt{s^2/n_{in}}$ for $\mW$ and $\sqrt{s^2/n_{out}}$ for $\vb$ where $s$ is a hyperparameter shared across all $i$'s. 
For all the experiments in \S\ref{sec:main_exp}, we set $s=\sqrt{0.5}$, based on the preliminary experiments in \S\ref{a_sec:perturbation_std}. 


In the cases of $\lambda$ and $\xi$, we mainly focus on selecting $\lambda$ while fixing $\xi$ to $1$, because these two loss functions, $\loss^{uni}$ and $\loss^{iod}$, are intertwined.
We use $\lambda=0.4$ for all the experiments in \S\ref{sec:main_exp}. With $\lambda=0.4$, we observed in the preliminary experiments that both $\loss^{uni}$ and $\loss^{sd}$ decrease. See \S\ref{a_sec:lambda_underclass}
for more details. 


\section{Experimental Settings}
\label{sec:exp_setup}

To evaluate our algorithm, we fine-tune deep neural networks on the various binary downstream tasks synthetically created out of existing dataset. Here, we describe the experimental setup.

\paragraph{Datasets and tasks}

We derive binary tasks from MNIST \citep{lecun1998gradient} 
, using the original labels. For example,
we can create a binary classification problem, distinguishing odd and even numbers from MNIST which originally has 10 classes (0-9 digits). In this way, we can create $2^{10}-2$ tasks from
MNIST
. After we define how to convert the original labels to either 0 or 1, we randomly select $N$ (for training) + $0.2N$ (for validation) instances, which allows us to test the impact of the size of labelled set.  
We standardize each image to have zero-mean and unit variance across all the examples. We do not use any data augmentation.

\paragraph{Models}

We train a multi-layer perceptron with fully-connected layers, $\fcn$, on MNIST
. $\fcn$ has three hidden layers with ReLU \citep{nair2010rectified} nonlinearity. 
$+\bn$ refers to the addition of batch normalization \citep{ioffe2015batch} to all hidden layers before ReLU. Additional details about the network architectures  are included in \S\ref{a_sec:model_archi}.

\paragraph{Baselines}

In order to assess the effectiveness of the proposed approach, we compare it against more conventional approaches to initialization. First, we compare our approach against data-agnostic initialization schemes, including Xavier initialization \citep{glorot2010understanding} and He initialization \citep{he2015delving}.
We also compare it to \textit{R.label} which refers to a data-dependent initialization scheme proposed by \citet{pondenkandath2018leveraging}. In the case of R.label, we randomly assign labels to the examples in  each mini-batch and minimize the cross entropy loss. Both our initial parameter configuration and R.label's initial parameter configuration are pre-trained on the same number of unlabelled examples for the same maximum number of epochs. For each pre-training run, we choose the parameter configuration based on the pre-training loss.
See \S\ref{a_sec:pretrain} 
for more details about the baselines and our pre-training setup.

Orthogonal to these initialization schemes, we also test adding batch normalization to these baseline approaches. It has been observed by some that batch normalization makes learning less sensitive to initialization \citep{ioffe2015batch}. 

\paragraph{Training and evaluation} 

For each initialization scheme, we fine-tune the network by minimizing the cross entropy loss, using Adam \citep{kingma2014adam} with a fixed learning rate of $10^{-3}$ and momentum parameters set to $(\beta_1,\beta_2)=(0.9,0.999)$. We use mini-batches of size 50 and train the network for up to 10 epochs without any regularization. For each binary task, we monitor the validation loss over the epochs and calculate the test accuracy (\%) on 10,000 test examples when the validation loss is at its minimum. We then report the mean and standard deviation of the test accuracy (\%) across 20 random binary tasks. We repeat this whole set of experiments four times, for each setup.

\section{Results}
\label{sec:main_exp}

\begin{table}[t]
\caption{We present the average ($\pm$stdev) 
test scores on MNIST across four random experiments by varying the number of labelled examples ($10N$ for training and $2N$ for validation). We denote the random label pre-training by \textit{R.label}. \textbf{Bold} marks the best score within each column. For all $N$, our initialization approximates various tasks better than the others do. Especially, when the number of labelled examples is small, the improvement is significant. Although both R.label and our initialization use 60,000 unlabelled data, our pre-training is superior to R.label. The positive effect of batch normalization ($+\bn$) can be observed with $N=40$, but its effect does not match that of our approach. Compared to $\fcn$ trained from scratch, we observe that $+\bn$ negatively impacts on the test score when the number of labelled instances is small ($N=5$) while our initialization improves the test score regardless of $N$.
} 
\label{tab:mnist_binary}
\begin{center}
    \begin{tabular}{ccc|cccc}
    \hline
    \multicolumn{1}{c}{\bf Model}&\multicolumn{1}{c}{\bf Init}& \multicolumn{1}{c|}{\bf Pre-trained} &\multicolumn{1}{c}{\bf N=5} &\multicolumn{1}{c}{\bf N=10} &\multicolumn{1}{c}{\bf N=20} &\multicolumn{1}{c}{\bf N=40}
    \\ 
    \hline 
    \hline
    \fcn&Xavier&Ours&\textbf{82.42}$\pm$0.72&85.98$\pm$0.65&\textbf{90.07}$\pm$0.17&92.48$\pm$0.57
    \\
    \fcn&Xavier&-&79.63$\pm$0.78&83.70$\pm$0.59&87.54$\pm$0.67&90.91$\pm$0.53
    \\
    \fcn&Xavier&R.label&76.81$\pm$2.13&83.34$\pm$0.79&87.53$\pm$0.91&90.88$\pm$0.52
    \\
    \fcn+\bn&Xavier&-&77.09$\pm$1.22&83.50$\pm$0.44&88.00$\pm$0.60&91.48$\pm$0.53
    \\
    \fcn+\bn&Xavier&R.label&78.87$\pm$1.75&84.38$\pm$0.97&88.71$\pm$0.53&91.57$\pm$0.59
    \\
    \hline
    \fcn&He&Ours&82.27$\pm$0.78&\textbf{86.46}$\pm$0.37&89.69$\pm$0.28&\textbf{92.61}$\pm$0.51
    \\
    \fcn&He&-&79.17$\pm$1.21&83.41$\pm$0.92&87.96$\pm$0.64&91.34$\pm$0.37
    \\
    \fcn&He&R.label&77.41$\pm$2.09&83.52$\pm$0.77&87.31$\pm$0.68&90.66$\pm$0.41
    \\
    \fcn+\bn&He&-&76.89$\pm$1.48&83.01$\pm$0.98&88.01$\pm$0.66&91.55$\pm$0.57
    \\
    \fcn+\bn&He&R.label&78.82$\pm$0.78&85.33$\pm$0.62&89.15$\pm$0.68&92.14$\pm$0.67
    \\ 
    \hline
\end{tabular}
\end{center}
\end{table}

\begin{table}[t]
\caption{We additionally demonstrate the standard deviation of test scores across 20 binary random tasks derived from MNIST by varying the number of labelled examples ($10N$ for training and $2N$ for validation). This metric measures the ability to solve most of tasks well (lower is better). We perform four random runs and report the average standard deviation. Here, ($\pm$stdev) means the standard deviation across four random experiments. We denote the random label pre-training by \textit{R.label}. \textbf{Bold} marks the best score within each column. Similar to Table \ref{tab:mnist_binary}, our initialization solves most of tasks well even if there are a small number of labelled examples. Both $+\bn$ and R.label can hurts the performance to approximate various tasks when the number of labelled instances is small (N=5).
} 

\label{tab:mnist_binary_std}
\begin{center}
    \begin{tabular}{ccc|cccc}
    \hline
    \multicolumn{1}{c}{\bf Model}&\multicolumn{1}{c}{\bf Init}& \multicolumn{1}{c|}{\bf Pre-trained} &\multicolumn{1}{c}{\bf N=5} &\multicolumn{1}{c}{\bf N=10} &\multicolumn{1}{c}{\bf N=20} &\multicolumn{1}{c}{\bf N=40}
    \\
    \hline 
    \hline
    \fcn&Xavier&Ours&\textbf{4.76}$\pm$0.88 & 4.54$\pm$0.52 & \textbf{3.01}$\pm$0.71 & 2.26$\pm$0.40
    \\
    \fcn&Xavier&-&6.62$\pm$1.29&5.55$\pm$0.62&3.54$\pm$0.27&2.65$\pm$0.53
    \\
    \fcn&Xavier&R.label&6.08$\pm$0.92&5.02$\pm$1.16&3.82$\pm$0.23&2.78$\pm$0.49
    \\
    \fcn+\bn&Xavier&-&7.47$\pm$1.70&5.53$\pm$0.69&3.72$\pm$0.78&2.72$\pm$0.32
    \\
    \fcn+\bn&Xavier&R.label&6.62$\pm$1.50&5.44$\pm$0.33&3.20$\pm$0.41&2.40$\pm$0.37
    \\
    \hline
    \fcn&He&Ours&5.26$\pm$0.87&\textbf{4.04}$\pm$0.80 & 3.25$\pm$0.42 & \textbf{2.16}$\pm$0.40
    \\
    \fcn&He&-&5.74$\pm$0.81&5.32$\pm$0.45&3.31$\pm$0.48&2.47$\pm$0.31
    \\
    \fcn&He&R.label&6.37$\pm$1.10&4.84$\pm$1.02&3.98$\pm$0.44&3.03$\pm$0.85
    \\
    \fcn+\bn&He&-&7.52$\pm$0.76&6.50$\pm$1.76&3.59$\pm$0.80&2.74$\pm$0.41
    \\
    \fcn+\bn&He&R.label&7.33$\pm$1.10&4.95$\pm$1.08&3.18$\pm$0.55&2.30$\pm$0.28
    \\ 
    \hline
\end{tabular}
\end{center}
\end{table}


Table~\ref{tab:mnist_binary} shows that the average test scores on 20 random binary tasks across 4 random runs. The 20 binary tasks for each run is the same regardless of model, initialization, and pre-training. Pre-training $\fcn$ with 60,000 unlabelled examples by our algorithm improves average test accuracy across 20 random tasks compared to that of training $\fcn$ from scratch, and this improvement is greater than the number of labelled instances is small. Furthermore, our test scores are better than all the schemes applied to $\fcn+\bn$ which has more parameters than $\fcn$. Both R.label and $+\bn$ bring the positive effect when the number of labelled examples is sufficient (N=40). However,  for $N=5$, both hurt the test performance of the randomly initialized plain network.

We also present the standard deviation of test scores across 20 binary random tasks created from MNIST in Table \ref{tab:mnist_binary_std}. Similar to Table \ref{tab:mnist_binary}, our initialization improves the ability to solve most of downstream tasks, and this improvement is greater when the number of labelled instances is small. We also observe R.label and $+\bn$ can hurt this ability in terms of the standard deviation for $N=5$. 

\section{Conclusion}
\label{sec:conclusion}

In this paper we proposed a novel criterion for identifying good initialization of parameters in deep neural networks. This criterion looks at the distribution over models derived from parameter configurations in the vicinity of an initial parameter configuration. If this distribution is close to a uniform distribution, the initial parameters are considered good, since we can easily reach any possible solution rapidly from there on. 

We then derived an unsupervised initialization algorithm based on this criterion. In addition to maximizing this uniformity,
our algorithm
prevents two degenerate cases; (1) degenerate softmax and (2) input-output detachment. 
Our experiments reveal that 
the model initialized by our algorithm can be trained better than the one trained from scratch, in terms of average test accuracy across a diverse set of tasks. This improvement was found to be comparable to or better than random label pre-training \citep{pondenkandath2018leveraging, maennel2020neural} and batch normalization \citep{ioffe2015batch} combined with typical initialization strategies. 

The effectiveness of the proposed approach leaves us with one puzzling question. The proposed algorithm does not take into account the use of gradient-based optimization, unlike model-agnostic meta-learning \citep{finn2017model}, and it could still find initial parameters that were amenable to gradient-based fine-tuning. This raises a question on the relative importance between initialization and the choice of optimizer in deep learning. We leave this question for the future.




\subsubsection*{Acknowledgments}
This work was supported by 42dot, Hyundai Motor Company (under the project Uncertainty in Neural Sequence Modeling), Samsung Advanced Institute of Technology (under the project Next Generation Deep Learning: From Pattern Recognition to AI), and NSF Award 1922658 NRT-HDR: FUTURE Foundations, Translation, and Responsibility for Data Science. This work was supported in part through the NYU IT High Performance Computing resources, services, and staff expertise.

\bibliography{iclr2023_conference}
\bibliographystyle{iclr2023_conference}


\clearpage
\appendix
\setcounter{section}{0}
\renewcommand\thesection{\Alph{section}}
\renewcommand\thesection{\Alph{section}}

\section{Appendix}

\subsection{Proofs for Theorem~\ref{thm:prob_out_of_nbd}}
\label{a_sec:thm1_proof}
To prove Theorem~\ref{thm:prob_out_of_nbd}, we introduce \textit{sub-exponential} random variables defined as follows:

\begin{definition}
\label{a_def:subexp}
    (Definition 2.7 in \citet{wainwright2019high}) A random variable $\rx$ is sub-exponential with $(\nu^2,\alpha)$, if
    \[
        \log \E\left[e^{\lambda(\rx-\E[\rx])}\right]\leq \frac{\lambda^2\nu^2}{2},~~\forall |\lambda|<\frac{1}{\alpha}.
    \]
    Furthermore, $\subexp(\nu^2,\alpha)$ refers to the collection of all sub-exponential random variables with $(\nu^2,\alpha)$. 
\end{definition}

For example, the chi-squared distribution with 1 degree of freedom, $\chi_1^2$, is sub-exponential with $(4,4)$ as stated in Lemma~\ref{a_lem:chi_is_subexp}:

\begin{lemma}
\label{a_lem:chi_is_subexp}
    If $\rx\sim\chi^2_1$, then $\rx\in\subexp(4,4)$.
\end{lemma}
\begin{proof}
    The moment generating function of $\rx\sim\chi_1^2$ and its expectation are well-known as
    \[
        \E[e^{\lambda\rx}] = \frac{1}{\sqrt{1-2\lambda}},~~\forall \lambda<\frac{1}{2},
    \]
    and $\E[\rx]=1$, respectively \citep{casella2021statistical}. Hence, we have
    \begin{align}
    \label{a_eq:chi_is_subexp}
        \log \E\left[e^{\lambda(\rx-\E[\rx])}\right] &= \log \frac{e^{-\lambda}}{\sqrt{1-2\lambda}}
        \\
        &=-\lambda-\frac{1}{2}\log(1-2\lambda)
        \nonumber
        \\
        &=-\lambda + \frac{1}{2}\sum_{n=1}^\infty \frac{(2\lambda)^n}{n}~~\left(\because\log(1-t)=-\sum_{n=1}^{\infty} \frac{t^n}{n},~~\forall |t|<1\right)
        \nonumber
        \\
        &=\frac{1}{2}\sum_{n=2}^\infty \frac{(2\lambda)^n}{n}
        \nonumber
        \\
        &=\frac{2\lambda^2}{2}\sum_{n=0}^\infty \frac{2}{n+2}(2\lambda)^n
        \nonumber
        \\
        &\leq \frac{2\lambda^2}{2}\sum_{n=0}^\infty (2\lambda)^n~~\left(\because\frac{2}{n+2}\leq 1,~~\forall n\geq0\right)
        \nonumber
        \\
        &=\frac{2\lambda^2}{2}\frac{1}{1-2\lambda},~~\forall|\lambda|<\frac{1}{2}~~\left(\because \frac{1}{1-t}=\sum_{n=0}^{\infty} t^n,~~\forall |t|<1\right).
        \nonumber
    \end{align}
    Since $\frac{1}{1-2\lambda}\leq 2$ for any $|\lambda|<\frac{1}{4}$, we have
    \[
        \log \E\left[e^{\lambda(\rx-\E[\rx])}\right]\leq \frac{4\lambda^2}{2},~~\forall |\lambda|<\frac{1}{4},
    \]
    form \eqref{a_eq:chi_is_subexp}. That is, $\rx\in\subexp(4,4)$ by Definition~\ref{a_def:subexp}.
\end{proof}

By combining Lemma~\ref{a_lem:scalar_mul_subexp} with Lemma~\ref{a_lem:ind_sum_subexp}, we show that linear combinations of independent sub-exponential random variables are also sub-exponential random variables. 

\begin{lemma}
\label{a_lem:scalar_mul_subexp}
    If $\rx\in\subexp(\nu^2, \alpha)$, then $\sigma^2\rx\in\subexp(\sigma^4\nu^2, \sigma^2\alpha)$.
\end{lemma}
\begin{proof}
    From $\rx\in\subexp(\nu^2,\alpha)$, we have
    \begin{align}
    \label{a_eq:scalar_mul_subexp}
        \log \E\left[e^{\lambda(\rx-\E[\rx])}\right]\leq \frac{\lambda^2\nu^2}{2},~~\forall |\lambda|<\frac{1}{\alpha},
    \end{align}
    by Definition~\ref{a_def:subexp}. By substituing $\sigma^2\lambda$ for $\lambda$ in \eqref{a_eq:chi_is_subexp}, we simply obtain
    \begin{align*}
    \label{a_eq:scalar_mul_subexp}
        \log \E\left[e^{\sigma^2\lambda(\rx-\E[\rx])}\right]&\leq \frac{(\sigma^2\lambda)^2\nu^2}{2},~~\forall |\sigma^2\lambda|<\frac{1}{\alpha},
    \end{align*}
    and this is equivalent to 
    \begin{align*}
        \log \E\left[e^{\lambda(\sigma^2\rx-\E[\sigma^2\rx])}\right]&\leq \frac{\lambda^2(\sigma^4\nu^2)}{2},~~\forall |\lambda|<\frac{1}{\sigma^2\alpha}.
    \end{align*}
    Therefore, $\sigma^2\rx\in\subexp(\sigma^4\nu^2,\sigma^2\alpha)$.
\end{proof}

\begin{lemma}
\label{a_lem:ind_sum_subexp}
    Suppose that $\rx_i\in\subexp(\nu_i^2, \alpha_i)$ for all $i=1,2,\cdots,m$. If $\rx_1,\rx_2,\cdots,\rx_m$ are mutually independent, then 
    \[
        \sum_{i=1}^m \rx_i\in\subexp\left(\sum_{i=1}^m \nu_i^2, \max_{i=1,2,\cdots,m}\alpha_i\right).
    \]
\end{lemma}
\begin{proof}
    Since $\rx_i\in\subexp(\nu_i^2,\alpha_i)$, we have
    \begin{align*}
        \log \E\left[e^{\lambda(\rx-\E[\rx])}\right]\leq \frac{\lambda^2\nu_i^2}{2},~~\forall |\lambda|<\frac{1}{\alpha_i}.
    \end{align*}
    by Definition~\ref{a_def:subexp}. By using the independence of $\rx_i$'s, we obtain
    \begin{align}\label{a_eq:ind_sum_subexp_nu}
        \log \E\left[e^{\lambda(\sum_{i=1}^m\rx_i-\E[\sum_{i=1}^m\rx_i])}\right]&=\log \E\left[e^{\sum_{i=1}^m(\lambda\rx_i-E[\rx_i])}\right]
        \\
        &=\log \E\left[\prod_{i=1}^m e^{\lambda\rx_i-E[\rx_i]}\right]
        \nonumber
        \\
        &=\log \prod_{i=1}^m \E\left[e^{\lambda\rx_i-E[\rx_i]}\right]
        \nonumber
        \\
        &=\sum_{i=1}^m \log \E\left[e^{\lambda\rx_i-E[\rx_i]}\right]
        \nonumber
        \\
        &\leq \frac{\lambda^2}{2}\sum_{i=1}^m\nu_i^2,~~\forall |\lambda|\in\bigcap_{i=1}^m\left\{\lambda: |\lambda|<\frac{1}{\alpha_i}\right\}.
        \nonumber
    \end{align}
    Note that 
    \begin{align}
    \label{a_eq:ind_sum_subexp_alpha}
        \left\{\lambda:|\lambda|< \frac{1}{\displaystyle\max_{i=1,2,\cdots,m} \alpha_i}\right\}=\bigcap_{i=1}^m\left\{\lambda: |\lambda|<\frac{1}{\alpha_i}\right\}.
    \end{align}
    By \eqref{a_eq:ind_sum_subexp_nu} and \eqref{a_eq:ind_sum_subexp_alpha}, we have 
    \[
        \log \E\left[e^{\lambda(\sum_{i=1}^m\rx_i-\E[\sum_{i=1}^m\rx_i])}\right] \leq \frac{\lambda\sum_{i=1}^m\nu_i^2}{2},~~\forall |\lambda|<\frac{1}{\displaystyle\max_{i=1,2,\cdots,m}\alpha_i}.
    \]
    Therefore, 
    \[
        \sum_{i=1}^m \rx_i\in\subexp\left(\sum_{i=1}^m \nu_i^2, \max_{i=1,2,\cdots,m}\alpha_i\right).
    \]
\end{proof}

The following Proposition~\ref{a_prop:subexp_tail_bdd} shows a tail bound of a sub-exponential random variable, and this proposition is the key for the proof of Theorem~\ref{thm:prob_out_of_nbd}. 

\begin{proposition}
\label{a_prop:subexp_tail_bdd}
    (Proposition 2.9 in \citet{wainwright2019high}) If $\rx\in\subexp(\nu^2,\alpha)$, then
    \[
        \sP(\rx-\E[\rx]\geq t) \leq \exp\left(-\frac{1}{2}\min\left\{\frac{t^2}{\nu^2},\frac{t}{\alpha}\right\}\right),
    \]
    for all $t>0$.
\end{proposition}
 
\begin{customthm}{\ref{thm:prob_out_of_nbd}}
    Let $\vtheta\sim\normal(\vtheta_0, \diag(\sigma_1^2,\sigma_2^2,\cdots,\sigma_m^2))$ and $\alpha_* =\max_{i=1,2,\cdots,m} \sigma_i^2$. If $r^2$ is greater than $m\alpha_*$, then we have
    \begin{align}
    \label{a_eq:prob_out_of_nbd}
        \sP\left(\|\vtheta-\vtheta_0\|\geq\ r \right) \leq \exp\left(-\frac{1}{8}\min\left\{\eta^2,m\eta\right\}\right),
    \end{align}
    where $\eta=\frac{r^2}{m\alpha_*}-1$.
\end{customthm}
\begin{proof}
    Since $\vtheta\sim\normal(\vtheta_0, \diag(\sigma_1^2,\sigma_2^2,\cdots,\sigma_m^2))$, $\vtheta-\vtheta_0$ follows $\normal(\vzero, \diag(\sigma_1^2,\sigma_2^2,\cdots,\sigma_m^2))$. We write $\vtheta-\vtheta_0$ as
    \[
        \vtheta-\vtheta_0 = (\sigma_1\rz_1, \sigma_2\rz_2,\cdots,\sigma_m\rz_m),
    \]
    where $\rz_i\sim\normal(0,1)$ for all $i=1,2,\cdots,m$ and $\rz_i$'s are mutually independent. Hence, we have
    \[
        \|\vtheta-\vtheta_0\|^2=\sum_{i=1}^{m}\sigma_i^2\rz_i^2,
    \]
    where $\rz_i^2\sim\chi_1^2$ for all $i=1,2,\cdots,m$ and $\rz_i^2$'s are mutually independent. From Lemma~\ref{a_lem:chi_is_subexp}, each $\rz_i^2$ is in $\subexp(4,4)$. Moreover, we have $\sigma_i^2 \rz_i^2\in\subexp(4\sigma_i^4, 4\sigma_i^2)$ by Lemma~\ref{a_lem:scalar_mul_subexp} and $\sum_{i=1}^m \sigma_i^2\rz_i^2\in\subexp(4\sum_{i=1}^m \sigma_i^4, 4\max_{i=1,2,\cdots,m}\sigma_i^2)$ by Lemma~\ref{a_lem:ind_sum_subexp}. We denote $\sum_{i=1}^m \sigma_i^4$ and  $\max_{i=1,2,\cdots,m}\sigma_i^2$ by $\nu_*^2$ and $\alpha_*$, respectively. Then, we have
    
    \[
        \|\vtheta-\vtheta_0\|^2\in\subexp(4\nu_*^2, 4\alpha_*),
    \]
    and
    \[
        \E[\|\vtheta-\vtheta_0\|^2]=\E\left[\sum_{i=1}^m\sigma_i^2\rz_i^2\right]=\sum_{i=1}^m\sigma_i^2\E[\rz_i^2]=\sum_{i=1}^m\sigma_i^2~~(\because \rz_i\sim\normal(0,1)).
    \]
    From Proposition~\ref{a_prop:subexp_tail_bdd}, we obtain a tail bound of $\|\vtheta-\vtheta_0\|^2$ as follows:
    \begin{align}\label{a_eq:prob_out_of_nbd_raw}
        \sP(\|\vtheta-\vtheta_0\|^2-\sigma_*^2\geq t) \leq \exp\left(-\frac{1}{2}\min\left\{\frac{t^2}{4\nu_*^2},\frac{t}{4\alpha_*}\right\}\right),
    \end{align}
    where $\sigma_*^2=\sum_{i=1}^m\sigma_i^2$  and $t>0$. Let $r^2=\sigma_*^2+t$. Then, we can rewirte \eqref{a_eq:prob_out_of_nbd_raw} as
    \begin{align}
    \label{a_eq:prob_out_of_nbd_bdd_raw}
        \sP(\|\vtheta-\vtheta_0\|^2 \geq r^2) \leq \exp\left(-\frac{1}{8}\min\left\{\frac{(r^2-\sigma_*^2)^2}{\nu_*^2},\frac{r^2-\sigma_*^2}{\alpha_*}\right\}\right).
    \end{align}
    Note that
    \begin{align}
    \label{a_eq:prob_out_of_nbd_bdd_nu}
        \nu_*^2 = \sum_{i=1}^m \sigma_i^4 \leq \left(\sum_{i=1}^m \sigma_i^2\right)^2=\sigma_*^4,
    \end{align}
    and
    \begin{align}
    \label{a_eq:prob_out_of_nbd_bdd_alpha}
        \sigma_*^2 = \sum_{i=1}^m \sigma_i^2 \leq m\alpha_*.
    \end{align}
    By \eqref{a_eq:prob_out_of_nbd_bdd_nu} and \eqref{a_eq:prob_out_of_nbd_bdd_alpha}, we have
    \begin{align}
    \label{a_eq:prob_out_of_nbd_bdd_nu_final}
        \frac{(r^2-\sigma_*^2)^2}{\nu_*^2} \geq \frac{(r^2-\sigma_*^2)^2}{\sigma_*^4}=\left(\frac{r^2}{\sigma_*^2}-1\right)^2, 
    \end{align}
    and
    \begin{align}
    \label{a_eq:prob_out_of_nbd_bdd_alpha_final}
        \frac{r^2-\sigma_*^2}{\alpha_*} \geq \frac{r^2-m\alpha_*}{\alpha_*}=\frac{r^2}{\alpha_*}-m.
    \end{align}
    Hence, we finally get
    \begin{align}
    \label{a_eq:prob_out_of_nbd_bdd_final}
        \sP(\|\vtheta-\vtheta_0\|^2 \geq r^2) &\leq \exp\left(-\frac{1}{8}\min\left\{\frac{(r^2-\sigma_*^2)^2}{\nu_*^2},\frac{r^2-\sigma_*^2}{\alpha_*}\right\}\right)
        \\
        &\leq \exp\left(-\frac{1}{8}\min\left\{\left(\frac{r^2}{\sigma_*^2}-1\right)^2,\frac{r^2}{\alpha_*}-m\right\}\right),
        \nonumber
    \end{align}
    from \eqref{a_eq:prob_out_of_nbd_bdd_raw} by using \eqref{a_eq:prob_out_of_nbd_bdd_nu_final} and \eqref{a_eq:prob_out_of_nbd_bdd_alpha_final}. Recall that we assume $r^2 \geq m\alpha_*$. By using this assumption and \eqref{a_eq:prob_out_of_nbd_bdd_alpha}, we have 
    \begin{align*}
        \frac{r^2}{\sigma_*^2}-1 \geq \frac{r^2}{m\alpha_*}-1\geq 0,
    \end{align*}
    and
    \begin{align}
    \label{a_eq:prob_out_of_nbd_min_final}
        \min\left\{\left(\frac{r^2}{\sigma_*^2}-1\right)^2,\frac{r^2}{\alpha_*}-m\right\}\geq \min\left\{\left(\frac{r^2}{m\alpha_*}-1\right)^2,\frac{r^2}{\alpha_*}-m\right\}.
    \end{align}
    Denote $\frac{r^2}{m\alpha_*}-1$ by $\eta$. Then, from \eqref{a_eq:prob_out_of_nbd_bdd_final} and \eqref{a_eq:prob_out_of_nbd_min_final}, we obtain
    \[
        \sP(\|\vtheta-\vtheta_0\|^2 \geq r^2)= \sP(\|\vtheta-\vtheta_0\| \geq r) \leq \exp\left(-\frac{1}{8}\min\left\{\eta^2,m\eta\right\}\right).
    \]
\end{proof}

\subsection{Proofs for Theorem~\ref{thm:prob_each_class}}
\label{a_sec:thm2_proof}
In the proof of Theorem~\ref{thm:prob_each_class}, we use Markov's inequality as stated below:
\begin{proposition}
\label{a_prop:markov_ineq} 
    (Markov's inequality, \citet{resnick2019probability}) If $\rx$ is a non-negative random variable with $\mathbb{E}[\rx]<\infty$, then 
        \[
            \sP(\rx\geq\lambda) \leq \frac{\E[\rx]}{\lambda},
        \]
    for any $\lambda>0$.
\end{proposition}
We also use the following Lemma~\ref{lem:dist_simplex_origin} to prove Theorem~\ref{thm:prob_each_class}.
\begin{lemma}
\label{lem:dist_simplex_origin}
    Let $\va$ be a vector in the $(d-1)$-dimensional unit simplex, $\Delta^{d-1}\subset\R^d$. Then,
    \begin{align}
    \label{eq:dist_simplex_origin}
        \min_{\va\in\Delta^{d-1}}\|\va\|^2=\frac{1}{d}.
    \end{align}
\end{lemma}
\begin{proof}
    Consider the minimization problem that
    \begin{align}\label{eq:extended_opt_prob}
        \argmin_{\va=(a_1,a_2,\cdots,a_d)\in R^d} &\|\va\|^2,
        \\
        \text{subject to~~~}&  \sum_{i=1}^d a_i = 1.
        \nonumber
    \end{align}
    By the method of Lagrange multipliers \citep{stewart2020calculus}, its minimum point satisfies  
    \begin{align}
    \label{eq:lagrangian}
        -\lambda \nabla_\va \|\va\|^2 = \nabla_\va \left(\sum_{i=1}^d a_i\right),
    \end{align}
    where $\lambda\in\R$. By solving \eqref{eq:lagrangian}, we have $a_i=-\frac{1}{2\lambda}$ for all $i$. Since $\sum_{i=1}^d a_i=1$, we have $\lambda = -\frac{d}{2}$ and $a_i=\frac{1}{d}$. Hence, the solution of \eqref{eq:extended_opt_prob}, $\va^*$, is
    \[
        \va^*=\left(\frac{1}{d}, \frac{1}{d},\cdots,\frac{1}{d}\right),
    \]
    and this is obviously in $\Delta^{d-1}$. That is, $\va^*=\argmin_{\va\in\Delta^{d-1}}\|\va\|^2$. Since $\|\va^*\|^2 = \frac{1}{d}$, \eqref{eq:dist_simplex_origin} is satisfied.
\end{proof}
We can finally prove Theorem~\ref{thm:prob_each_class} by using Proposition~\ref{a_prop:markov_ineq} and Lemma~\ref{lem:dist_simplex_origin}.
\begin{customthm}{\ref{thm:prob_each_class}}
    Let $\vv^{(i)}=\left(v^{(i)}_1,v^{(i)}_2,\cdots,v^{(i)}_d\right)\in\Delta^{d-1}$ where $v^{(i)}_i=1$ and $\sA_i$ be a subset of $\Delta^{d-1}$, defined in \eqref{def:ith_part_simplex}. Then, 
    \begin{align}
    \label{a_eq:lower_prob_each_class}
        \sP_{\vx\sim p(\vx)}(\vf_{\vtheta_0^*+\vepsilon}(\vx)\in\sA_i) \geq 1-\sqrt{d}\E_{\vx\sim p(\vx)}[\|\vv^{(i)}-\vf_{\vtheta_0^*+\vepsilon}(\vx)\|],
    \end{align}
    for a given $\vepsilon\sim\normal(\vzero,\mSigma)$.
\end{customthm}
\begin{proof}
    Let $\sV_i = \left\{\va=(a_1,a_2,\cdots,a_d)\in\Delta^{d-1}:a_i\geq\frac{1}{2}\right\}$. By the definition of $\sA^{(i)}$, we have $\sV_i\subset\sA_i$. Furthermore,
    \begin{align}
        \|\vv^{(i)}-\va\|^2 &= \sum_{j=1}^d (v^{(i)}_j - a_j)^2\label{eq:dist_from_vertex}
        \\
        &= (1-a_i)^2 + \sum_{j\neq i} a_j^2
        \nonumber
        \\
        &= 1-2a_i + \sum_{j=1}^d a_j^2.
        \nonumber
    \end{align}
    
    Suppose that $\sW_i=\{\va\in\Delta^{d-1}:\|\vv^{(i)}-\va\| < \frac{1}{\sqrt{d}}\}$. By using \eqref{eq:dist_from_vertex}, we obtain
    \begin{align}
        \va\in\sW_i &\Longleftrightarrow 1-2a_i+\sum_{j=1}^d a_j^2 < \frac{1}{d}
        \nonumber
        \\
        &\Longleftrightarrow a_i> \frac{1}{2}+\frac{1}{2}\sum_{j=1}^d a_j^2 -\frac{1}{2d}
        \nonumber
        \\
        &\Longrightarrow a_i> \frac{1}{2}~~\left(\because \min_{\va\in\Delta^{d-1}}\|\va\|^2=\frac{1}{d}\textrm{ ~by Lemma~\ref{lem:dist_simplex_origin}}\right).
        \nonumber
    \end{align}
    Hence, $\va\in\sW_i$ implies that $\va\in\sV_i$, and we have $\sW_i\subset\sV_i\subset\sA_i$. By Markov's inequality in Proposition~\ref{a_prop:markov_ineq}, we obtain
    \begin{align}
        \sP_{\vx\sim p(\vx)}(\vf_{\vtheta_0^*+\vepsilon}(\vx)\in\sA_i) &\geq \sP_{\vx\sim p(\vx)}(\vf_{\vtheta_0^*+\vepsilon}(\vx)\in\sW_i)
        \nonumber
        \\
        &= \sP_{\vx\sim p(\vx)}\left(\|\vv^{(i)}-\vf_{\vtheta_0^*+\vepsilon}\|<\frac{1}{\sqrt{d}}\right)
        \nonumber
        \\
        &=1-\sP_{\vx\sim p(\vx)}\left(\|\vv^{(i)}-\vf_{\vtheta_0^*+\vepsilon}\|\geq\frac{1}{\sqrt{d}}\right)
        \nonumber
        \\
        &\geq 1-\sqrt{d}\E_{\vx\sim p(\vx)}[\|\vv^{(i)}-\vf_{\vtheta_0^*+\vepsilon}(\vx)\|].
        \nonumber
    \end{align}
\end{proof}

\subsection{Proofs for Theorem~\ref{thm:he_init}}
\label{a_sec:thm3_proof}

\begin{proposition}
\label{a_prop:slln}
    (Kolmogorov's strong law of large numbers, Theorem 7.5.1 in \citet{resnick2019probability}) Let $\{\rx_1,\rx_2,\cdots\}$ be random samples from $\rx$. If $\E[\rx]<\infty$, then
    \[
        \frac{1}{n}\sum_{i=1}^n \rx_i \rightarrow \E[\rx],
    \]
    as $n\rightarrow\infty$.
\end{proposition}

\begin{customthm}{\ref{thm:he_init}}
    Let $\vf_{\vtheta_0}$ be a fully connected network with ReLU \citep{nair2010rectified} non-linearity. We write the layerwise non-linear transformation from $\vx_l$ to $\vx_{l+1}$ for $l\neq 0$ as 
    \begin{align}
    \label{a_eq:single_layer_transform}
        \vf_{\vtheta_0^{(l:l+1)}}(\vx_l)
        =
        \mW^{(l+1)}\relu(\vx_l)+\vb^{(l+1)},
    \end{align}
    where $\mW^{(l+1)}\in\R^{n_{l+1}\times n_l}$ is the weight matrix and $\vb\in\R^{n_{l+1}}$ is the bias vector. 
    Assume that each element of $\vx_l$ has a symmetric distribution at $0$ and all elements of $\vx_l$ are mutually independent. If the $(i,j)$-th entry of $\mW^{(l+1)}$, $W_{ij}^{(l+1)}$, is a random sample from $\normal(0,\sigma_l^2)$ and $\vb^{(l+1)}$ is $\vzero$, then the following equality holds for all $k=1,2,\cdots, n_{l+1}$ when $\sigma_l=\sqrt{\frac{2}{n_l}}$ with sufficiently large $n_l$:
    \begin{align}
    \label{a_eq:he_jaco_condition}
        1\approx\left\|J_{\vtheta_0^{(l:l+1)}}^{(k)}(\vx_l)\right\|_F
        =
        \|\mW^{(l+1)} \mathds{1}(\vx_l > 0)\|_F, 
    \end{align}
    where $\mathds{1}(\vx_l >0)$ turns each positive entry in $\vx_l$ to $1$ and $0$ otherwise.
\end{customthm}
\begin{proof}
    Denote the $i$th element of $\vx_l$ by $x_{l,i}$. Simply, the $i$th row of Jacobian of \eqref{a_eq:single_layer_transform} with respect to $\vx_l$ is $W^{(l+1)}_{i\cdot}\mathds{1}(\vx_l > 0)$ where $W^{(l+1)}_{i\cdot}$ is the $i$th row of $\mW^{(l+1)}$ and
    $\mathds{1}(\vx_l >0)=(\vone(x_{l,1}),\vone(x_{l,2}),\cdots,\vone(x_{l,n_{l}}))$ satisfying
    \[
        \vone(x)= \begin{cases}
        1,&\text{ if }x > 0, \\
        0,&\text{ otherwise}.
        \end{cases}
    \]
    If we assume that $x_{l,i}$ is sampled from a symmetric distribution at 0 and $x_{l,i}$'s are mutually independent for all $i=1,2,\cdots,n_l$,\footnote{This condition is the same as \citet{he2015delving} assumed.} then $\vone(x_{l,i})$'s are random samples from $\bernoulli(0.5)$.
    
    Since $W^{(l+1)}_{ij}$'s are random samples from $\normal(\vzero,\sigma_l^2)$, $W^{(l+1)}_{ij}\vone(x_{l,j})$ are random samples from $\bernoulli(0.5)\times\normal(\vzero,\sigma_l^2)$ for all $j=1,2,\cdots,n_l$. By using the strong law of large number in Proposition~\ref{a_prop:slln}, we have
    \begin{align}
    \label{a_eq:he_init_jacobian}
        \left\|J_{\vtheta_0^{(l:l+1)}}^{(k)}(\vx_l)\right\|_F^2&=\left\|W_{\cdot}^{(l+1)}\mathds{1}(\vx_l >0)\right\|_F^2
        \\
        &= n_l\times \frac{1}{n_l}\sum_{j=1}^{n_l}\left\{W^{(l+1)}_{ij}\vone(x_{l,j})\right\}^2\approx n_l\E[\rx^2\ry^2],
        \nonumber 
    \end{align}
    where $\rx\sim\normal(0,\sigma_l^2)$ and $\ry\sim\bernoulli(0.5)$ for sufficiently large $n_l$.
    
    Since $\rx$ and $\ry$ are independent, we obtain
    \begin{align}
    \label{a_eq:slln_limit}
        \E[\rx^2\ry^2]=\E[\rx^2]\E[\ry^2]=\frac{1}{2}\sigma_l^2. 
    \end{align}
    From \eqref{a_eq:he_init_jacobian} and \eqref{a_eq:slln_limit},
    \[
        \sigma_l = \sqrt{\frac{2}{n_l}},
    \]
    implies that 
    \[
        \left\|J_{\vtheta_0^{(l:l+1)}}^{(k)}(\vx_l)\right\|_F^2\approx 1.
    \]
\end{proof}

\section{Experiments for \S\ref{sec:theory}}
\label{a_sec:theory_exp}

In this section, we empirically validate our claims discussed in \S\ref{sec:degenerate_remedy}, \S\ref{sec:input_ignoring}, and \S\ref{sec:final_loss}. For this, we pre-train $\fcn$ (a 784-392-392-392-2 fully connected network) described in \S\ref{a_sec:model_archi}, using unlabelled MNIST examples. Before pre-training, we radomly initialize our network by Xavier initialization \citep{glorot2010understanding}. Unless explicitly stated, all learning hyperparameters for pre-training, such as an optimizer and a learning rate, are the same as the hyperparameters presented in \textbf{our pre-training} of \S\ref{a_sec:pretrain}.

\subsection{Real cases for the degenerate softmax and the input-output detachment} \label{a_sec:degenerate_cases_exp}

We insisted that minimizing \eqref{eq:uniformity_loss} with respect to $\vtheta_0$ can make each perturbed model near $\vtheta_0$ be a constant function as shown in \eqref{eq:degenerate_ex} of \S\ref{sec:degenerate_remedy}. To empirically validate this, we pre-train $\fcn$ by minimizing \eqref{eq:final_loss} with $\lambda=\xi=0$.

Figure \ref{a_fig:degenerate_ex} shows that minimizing only $\loss^{uni}$ can cause the degenerate softmax and the input-output detachment. In Figure \ref{a_fig:degenerate_ex}~(a), we demonstrate that minimizing $\loss^{uni}$ properly encourages our perturbed model prediction to be uniform over the $\uniform(\Delta^{d-1})$ for a given $\vx\sim p(\vx)$. However, we observe that each perturbed model becomes a constant function after pre-training as shown in Figure \ref{a_fig:degenerate_ex}~(b). It means that our pre-trained model predicts regardless of inputs (the input-output detachment). Moreover, the perturbed model of the third plot in Figure \ref{a_fig:degenerate_ex}~(b) collapses into $(0.9,0.1)\in\Delta^{2-1}$ for most inputs. In other words, this model classifies all instances into class 0, and this is the case of the degenerate softmax. We indeed need additional regularization terms to avoid both the input-output detachment and the degenerate softmax when minimizing $\loss^{uni}$.

\begin{figure}[t]
    \begin{center}
    \subfigure[$\vf_{\vtheta_0+\vepsilon}(\vx)$ given $\vx\sim p(\vx)$]{\includegraphics[trim=0cm 0cm 0cm 0.9cm, clip,width=\linewidth]{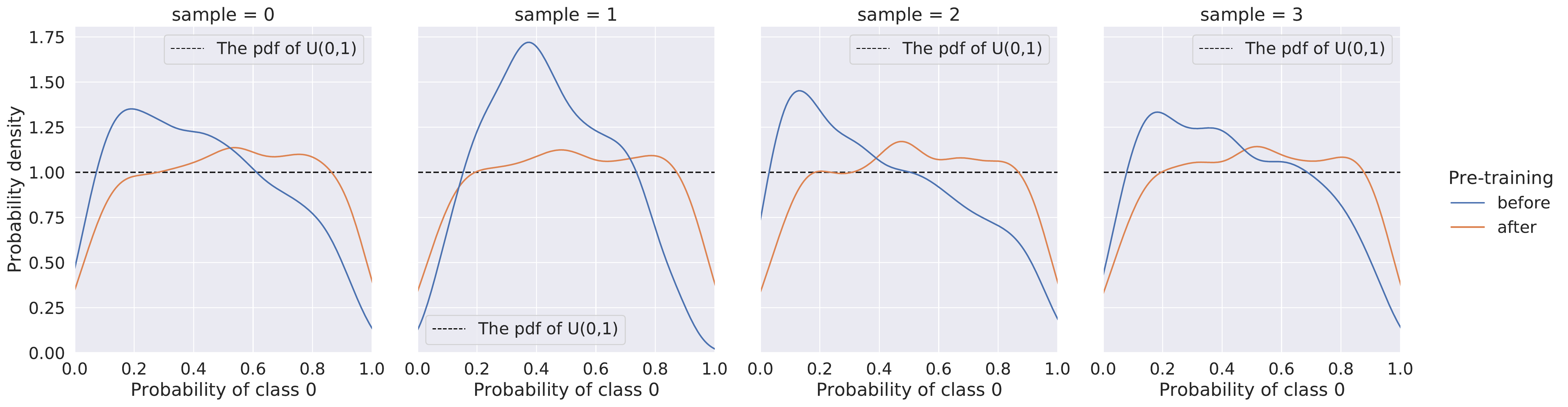}}
    \subfigure[$\vf_{\vtheta_0+\vepsilon}(\vx)$ given $\vepsilon\sim\normal(\vzero, \mSigma)$]{\includegraphics[trim=0cm 0cm 0cm 0.9cm, clip,width=\linewidth]{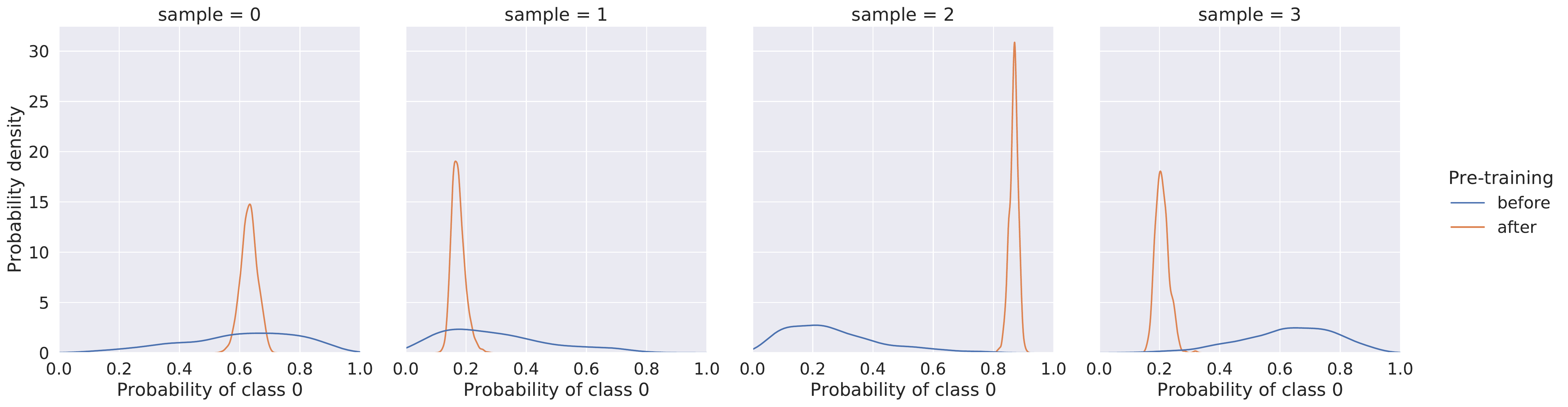}}
	\end{center}
	\caption{Effect of pre-training $\fcn$ on MNIST by minimizing only $\loss^{uni}$ with respect to $\vtheta_0$. Since $\fcn$ is a model for binary classification, we present the probability of class 0 on $[0,1]$ instead of the model prediction on $\Delta^{2-1}$. We note that the probability of class 0 follows $\uniform(0,1)$ if and only if the model prediction is equal to $\uniform(\Delta^{2-1})$ in distribution. (a): We randomly sample $\vx_i$ from $p(\vx)$ and present the pdf of the first element (the probability of class 0) of $\vf_{\vtheta_0+\vepsilon}(\vx_i)$ for $i=1,2,3,4$. For each density plot, we use 1,024 random perturbations from $\normal(\vzero,\mSigma)$. Compared to the randomly initialized network (blue), the distribution of model predictions given $\vx\sim p(\vx)$ for the pre-trained network (orange) converges to $\uniform(\Delta^{2-1})$ (black dotted) in distribution. (b): We demonstrate the pdf of the first element of $\vf_{\vtheta_0+\vepsilon_i}(\vx)$ for $i=1,2,3,4$. Each density function uses 1,024 random instances from $p(\vx)$. After pre-training, the predictions of each perturbed model collapsed into a constant regardless of inputs.} \label{a_fig:degenerate_ex}
\end{figure}

\subsection{Impact of \eqref{eq:ignore_loss} on nearby models}
\label{a_sec:nearby_jaco_exp}
To prevent the input-output detachment, we proposed $\loss^{iod}$ in \eqref{eq:ignore_loss}. By minimizing $\loss^{iod}$ with respect to $\vtheta_0$, $\left\|J^{(i)}_{\vtheta_0^{(l:L)}}(\vx_l) \right\|_F$ converges to 1 for all $i=1,2,\cdots,d$ and $l=0,1,\cdots,L-1$. In other words, $\vf_{\vtheta_0}(\vx)$ is not a constant function of $\vx$. By the continuity of $\left\|J^{(i)}_{\vtheta_0^{(l:L)}}(\vx_l) \right\|_F$ with respect to $\vtheta_0$, we have $\left\|J^{(i)}_{{(\vtheta_0+\vepsilon)}^{(l:L)}}(\vx_l) \right\|_F$ is approximately equal to $\left\|J^{(i)}_{\vtheta_0^{(l:L)}}(\vx_l) \right\|_F$ when $\|\vepsilon\|^2$ is sufficiently small. It means that $\vf_{\vtheta_0+\vepsilon}(\vx)$ is not constant function of $\vx$ as well.

We empirically validate that minimizing $\loss^{iod}(\vtheta_0)$ with respect to $\vtheta_0$ encourages $\vf_{\vtheta_0+\vepsilon}(\vx)$ not to converge a constant function of $\vx$. In Figure \ref{a_fig:iodloss_nearby}, $\loss^{iod}(\vtheta_0+\vepsilon)$ (orange) decreases as we minimizes $\loss^{iod}(\vtheta_0)$ (blue) where $\vepsilon\sim\normal(\vzero,\mSigma(s))$.\footnote{We recommended the choice of $\mSigma$ for $\fcn$, which is determined by both $s$ and the dimensionality of each layer, in \S\ref{sec:final_loss}. $\mSigma(s)$ refers to this recommendation for $\mSigma$.} Since $\|\vepsilon\|$ is proportional to $s$, the effect of minimizing $\loss^{iod}(\vtheta_0)$ on perturbed models is diminished as $s$ increases. 

\begin{figure}[t]
    \begin{center}
    \includegraphics[trim=0cm 0cm 0cm 0.9cm, clip,width=\linewidth]{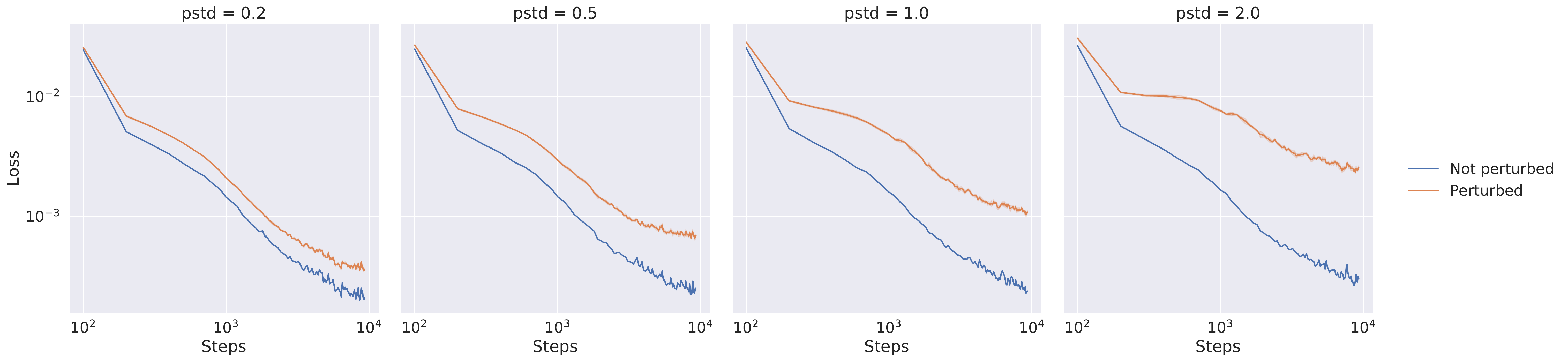}
	\end{center}
	\hspace{1.1cm}(a)~$s=\sqrt{0.2}$\hspace{1.1cm}(b)~$s=\sqrt{0.5}$\hspace{1.1cm}(c)~$s=\sqrt{1.0}$\hspace{1.1cm}(d)~$s=\sqrt{2.0}$
	\caption{We present $\loss^{iod}(\vtheta_0)$ (blue) and $\loss^{iod}(\vtheta_0+\vepsilon)$ (orange), as the function of optimization steps in log-log scale, when we pre-train $\fcn$ on MNIST by minimizing $\loss^{iod}(\vtheta_0)$ with respect to $\vtheta_0$. For $\loss^{iod}(\vtheta_0+\vepsilon)$, we report mean (curve) $\pm$ stdev (shaded area) across four random perturbations from $\normal(\vzero, \mSigma(s)$). $\mSigma(s)$ is described in \S\ref{sec:final_loss}. For small $s$, such as $s=\sqrt{0.2}$ and $s=\sqrt{0.5}$, minimizing $\loss^{iod}(\vtheta_0)$ not only prevents $\vf_{\vtheta_0}(\vx)$ from being constant but also prevents $\vf_{\vtheta_0+\vepsilon}(\vx)$ from being constant.}\label{a_fig:iodloss_nearby}
\end{figure}

\subsection{Necessity of all three loss functions}
\label{a_sec:need_all_exp}
Our final loss function in \eqref{eq:final_loss} is composed of $\loss^{uni}$, $\loss^{sd}$, and $\loss^{iod}$. In \S\ref{a_sec:degenerate_cases_exp}, we empirically demonstrated that minimizing only $\loss^{uni}$ with respect to $\vtheta_0$ ($\lambda=\xi=0$) causes both the degenerate softmax and the input-output detachment. We designed $\loss^{sd}$ and $\loss^{iod}$ to avoid the degenerate softmax and the input-output detachment, respectively. In this section, we empirically validate the effectiveness of $\loss^{sd}$ and $\loss^{iod}$.

\subsubsection{Effectiveness of \eqref{eq:underclass_loss}}
According to our analysis in \S\ref{sec:under-class}, minimizing \eqref{eq:final_loss} with $\lambda>0$ encourages each perturbed model to classify inputs into all $d$ classes. We therefore pre-train $\fcn$ on MNIST with $(\lambda, \xi)\in\{(0.0,0.0), (0.4,0.0), (0.0,1.0), (0.4,1.0)\}$, and count the number of perturbed models that have the degenerate softmax on a mini-batch of size 32. In other words, if the perturbed model classifies all 32 instances in the mini-batch into $k$ classes where $k<d$, we consider that this perturbed model has the degenerate softmax. We report the mean and standard deviation of the counts (\%) out of 256 perturbed models across 128 random mini-batches.    

\begin{table}[h]
\caption{We demonstrate the average ($\pm$stdev) ratio (\%) of perturbed models with the degenerate softmax to 256 perturbed models across 128 mini-batches of size 32. Without positive $\lambda$, most of perturbed models has the degenerate softmax for both $\xi=0.0$ and $\xi=1.0$. $\loss^{sd}$ effectively helps our pre-training to avoid the issue of degnerate softmax.} 
\label{a_tab:effect_loss_sd}
\begin{center}
    \begin{tabular}{cc|c}
    \hline
    \multicolumn{1}{c}{$\boldsymbol{\lambda}$}&\multicolumn{1}{c|}{$\boldsymbol{\xi}$}&\multicolumn{1}{c}{\bf Ratio (\%)}
    \\
    \hline 
    \hline
    0.0&0.0&89.32$\pm$0.94\\
    0.4&0.0&~~0.41$\pm$0.38\\
    0.0&1.0&80.83$\pm$1.03\\
    0.4&1.0&~~1.39$\pm$0.62\\
    \hline
\end{tabular}
\end{center}
\end{table}

Table \ref{a_tab:effect_loss_sd} shows that minimizing $\loss^{sd}$ ($\lambda=0.4$) greatly reduces the number of perturbed models that have the degenerate softmax compared to $\lambda=0.0$. With or without minimizing $\loss^{iod}$ (either $\xi=1.0$ or $\xi=0.0$), the minimization of $\loss^{sd}$ ($\lambda=0.4$) makes each perturbed model classify inputs into all $d$ classes. To prevent the degenerate softmax, we need positive $\lambda$ for our final loss function in \eqref{eq:final_loss}.

\subsubsection{Effectiveness of \eqref{eq:ignore_loss}}

We designed $\loss^{iod}$ in \S\ref{sec:input_ignoring} to alleviate the issue of input-output detachment. In this section, we analyze why this issue matters and how $\loss^{iod}$ remedies it.
Let $\vtheta_0^*$ be an optimal solution of minimizing \eqref{eq:final_loss} with $\xi=0$, then $\vtheta_0^*$ enables our model to find $\vf_{\vtheta^*}$ approximating a given target mapping $\vf^*$ where $\vtheta^*=\vtheta_0^*+\vepsilon^*$ and $\vepsilon^*\sim\normal(\vzero,\mSigma)$ with $\mSigma=\diag(\sigma_1^2,\sigma_2^2,\cdots,\sigma_m^2)$. However, the existence of good approximation of $\vf^*$ in the neighborhood of $\vtheta_0^*$ does not guarantees a trajectory from $\vtheta_0^*$ to $\vtheta^*$, which is movable by gradient descent. Specifically, we assume that $\vf_{\vtheta^*}$ around $\vf_{\vtheta^*_0}$ approximates $\vf^*$ by adding Gaussian perturbation with a variance of $\sigma_i^2$ to $\theta^*_{0,i}$ where $\theta^*_{0,i}$ is the $i$th parameter of $\vtheta^*_0$. This assumption implies that, at $\vtheta_0^*$, the gradient descent for approximating $\vf^*$ should be able to update each parameter $\theta^*_{0,i}$ with a specific level of strength proportional to $\sigma_i^2$.

However, $\vtheta_0^*$ can be an initial parameter configuration where some parameters of $\vtheta_0^*$ be hardly changed by the gradient descent. For instance, suppose that our model is $\fcn$ parametrized by $\vtheta_0^*$. If each bias parameter of $\vtheta_0^*$ at the first hidden layer is a large negative number, then, no matter what the input is, most of ReLU units at the first hidden layer are dead. We call a neuron that dies for all inputs \text{a fully dead neuron}. It means that all parameters going into the fully dead neurons cannot be updated by back-propagation \citep{rumelhart1986learning}. Therefore, our model initialized by $\vtheta_0^*$ cannot reach $\vtheta^*$ satisfying $|\theta^*_j-\theta^*_{0,j}|>0$ where $\theta^*_{0,j}$ is one of parameters connected into the fully dead neurons.\footnote{$\theta^*_j$ and $\theta_{0,j}^*$ are the $j$th entries of $\vtheta^*$ and $\vtheta^*_0$, respectively.}

Even if all ReLU units at the first hidden layer die for all inputs, our model is still able to predict uniformly over $\Delta^{d-1}$ due to perturbations in the next layers. Particularly, our model can consider each perturbation added to bias parameters of the second hidden layer as a new input for all layers after the second hidden layer. After that, our model learns how to generate various predictions near $\vtheta_0$, which are completely detached from the original input distribution $p(\vx)$. In order to check whether this case occurs in $\vtheta_0^*$, we pre-train $\fcn$ on MNIST with $(\lambda, \xi)\in$\{(0.0,~0.0), (0.4,~0.0), (0.4,~1.0)\}. After pre-training, we count the number of fully dead neurons on each mini-batch of size 32 at each hidden layer. We report the average and standard deviation of these counts (\%) out of the number of all neurons at each hidden layer across 128 random mini-batches.

\begin{table}[h]
\caption{We present the average ($\pm$stdev) ratio (\%) of fully dead neurons at each hidden layer across 128 mini-batches of size 32. Hidden $i\in\{1,2,3\}$ refers to the $i$-th hidden layer. With $\lambda=0.0$ and $\xi=0.0$, more than 50\% of neurons are fully dead. Pre-training with $\lambda=0.4$ and $\xi=0.0$ greatly reduces the number of fully dead neurons except the second hidden layer (Hidden 2). When we use both $\lambda=0.4$ and $\xi=1.0$, there are a small number of fully dead neurons for all hidden layers.} 
\label{a_tab:effect_loss_iod}
\begin{center}
    \begin{tabular}{cc|ccc}
    \hline
    \multicolumn{1}{c}{$\boldsymbol{\lambda}$}&\multicolumn{1}{c|}{$\boldsymbol{\xi}$}&\multicolumn{1}{c}{\bf Hidden 1 (\%)}&\multicolumn{1}{c}{\bf Hidden 2 (\%)}&\multicolumn{1}{c}{\bf Hidden 3 (\%)}
    \\
    \hline 
    \hline
    0.0&0.0&52.11$\pm$3.07&63.88$\pm$1.30&80.08$\pm$0.61\\
    0.4&0.0&~~1.96$\pm$0.42&21.81$\pm$1.40&8.21$\pm$2.16\\
    0.4&1.0&~~3.75$\pm$0.77&~~5.12$\pm$1.20&~~0.02$\pm$0.06\\
    \hline
\end{tabular}
\end{center}
\end{table}

In Table \ref{a_tab:effect_loss_iod}, pre-training without minimizing both $\loss^{sd}$ and $\loss^{iod}$ ($\lambda=0.0$ and $\xi=0.0$) causes a lot of fully dead neurons for all hidden layers. For example, 52.11\% of the neurons are fully dead in the first hidden layer, and the number of fully dead neurons at the $i$-th hidden layer increases as $i$ increases. Using $\lambda=0.4$ and $\xi=0.0$ greatly reduces the number of fully dead neurons at the first and third hidden layers, but there are still 21.82\% of neurons that are fully dead at the second layer. It means that the input-output detachment occurs at the second hidden layer for the pre-trained model using $\lambda=0.4$ and $\xi=0.0$. With $\lambda=0.4$ and $\xi=1.0$, the number of fully dead neurons are evenly small for all hidden layers. This shows that minimizing $\loss^{iod}$ effectively prevents the input-output detachment.    


\section{Choice of hyperparameters}
\label{a_sec:hyperparams}
For all experiments in this section, we pre-train $\fcn$ (a 784-392-392-392-2 fully connected network) described in \S\ref{a_sec:model_archi}, using unlabelled MNIST examples. Before pre-training, we radomly initialize our network by Xavier initialization \citep{glorot2010understanding}. Unless explicitly stated, all learning hyperparameters for pre-training, such as an optimizer and a learning rate, are kept same as in \textbf{our pre-training} of \S\ref{a_sec:pretrain}.

\subsection{Median heuristic for MMD}
\label{a_sec:med_heuristic}
We use the median heuristic \citep{smola1998learning} for the bandwidth of MMD in $\loss^{uni}$, $\gamma$. In equation \eqref{eq:mmd_normal_nbd}, we compute three types of kernel embedding given $\vx\sim p(\vx)$: 
\begin{itemize}
    \item $\E_{\vepsilon\sim\normal(\vzero,\mSigma)}\E_{\vepsilon'\sim\normal(\vzero,\mSigma)}[k_\gamma(\vf_{\vtheta_0+\vepsilon}(\vx), \vf_{\vtheta_0+\vepsilon'}(\vx))]$
    \item $\E_{\vepsilon\sim\normal(\vzero,\mSigma)}\E_{\vu\sim \uniform(\Delta^{d-1})}[k_\gamma(\vf_{\vtheta_0+\vepsilon}(\vx), \vu)]$
    \item $\E_{\vu\sim \uniform(\Delta^{d-1})}\E_{\vu'\sim \uniform(\Delta^{d-1})}[k_\gamma(\vu, \vu')]$,
\end{itemize}
where $k_\gamma(\vx,\vy)=\exp\left(-\frac{\|\vx-\vy\|^2}{2\gamma^2}\right)$. To compute them, we first sample $\vepsilon_1,\vepsilon_1,\cdots,\vepsilon_M$ from $\normal(\vzero,\mSigma)$ and $\vu_1,\vu_2,\cdots,\vu_N$ from $\uniform(\Delta^{d-1})$. We then need to calculate three types of pairwise distances corresponding to the list above:
\begin{itemize}
    \item $\|\vf_{\vtheta_0+\vepsilon_i}(\vx)- \vf_{\vtheta_0+\vepsilon_j}(\vx)\|$ for all $1\leq i < j \leq M$
    \item $\|\vf_{\vtheta_0+\vepsilon_i}(\vx)-\vu_j\|$ for all $1\leq i \leq M$ and $1\leq j \leq N$
    \item $\|\vu_i- \vu_j\|$ for all $1\leq i < j \leq N$
\end{itemize}
The median heuristic sets $\gamma$ to the median of total $\binom{M}{2}+MN+\binom{N}{2}$ pairwise distances, $\gamma_{med}$. Based on the median heuristic, we finally use the sum of various kernel functions defined by 
\[
    k(\vx,\vy) = \sum_{i=-4}^{4} k_{\gamma_i}(\vx,\vy), 
\]
where $\gamma_i=2^{i}\times\gamma_{med}$.

\subsection{Standard deviation of perturbation}
\label{a_sec:perturbation_std}
As discussed in \S\ref{sec:final_loss}, we need to determine $s$ and $\lambda$ for our final loss in \eqref{eq:final_loss}. Our final loss function is composed of three loss functions: $\loss^{uni}$, $\loss^{sd}$, and $\loss^{iod}$. Minimizing our main loss function, $\loss^{uni}$, makes our model predictions be evenly spread over $\Delta^{d-1}$, so $\loss^{uni}$ should decrease to 0 during pre-training. $\loss^{sd}$ prevents the degenerate softmax, and we can measure its effect by the average ratio (\%) of perturbed models with the degenerate softmax on each mini-batch as shown in Table \ref{a_tab:effect_loss_sd}. $\loss^{iod}$ regularizes our pre-training to alleviate the issue of input-output detachment, and we are able to figure out its effect by the average ratio (\%) of fully dead neurons at each hidden layer as demonstrated in Table \ref{a_tab:effect_loss_iod}. Since the average ratio of fully dead neurons need to be 0 for all hidden layers, we evaluate the effect of $\loss^{iod}$ by their maximum ratio.  

To empirically find optimal $s$ and $\lambda$, we pre-train $\fcn$ on MNIST with various pairs of $s\in\{\sqrt{0.2}, \sqrt{0.5}, \sqrt{1.0}, \sqrt{2.0}\}$ and $\lambda\in\{0.2, 0.4, 1.0\}$ while fixing $\xi$ to 1. For each pair, we perform four random experiments.
We first monitor the behavior of $\loss^{uni}$ during pre-training with each setup.

\begin{figure}[h]
    \begin{center}
    \includegraphics[trim=0cm 0cm 0cm 0.9cm, clip,width=\linewidth]{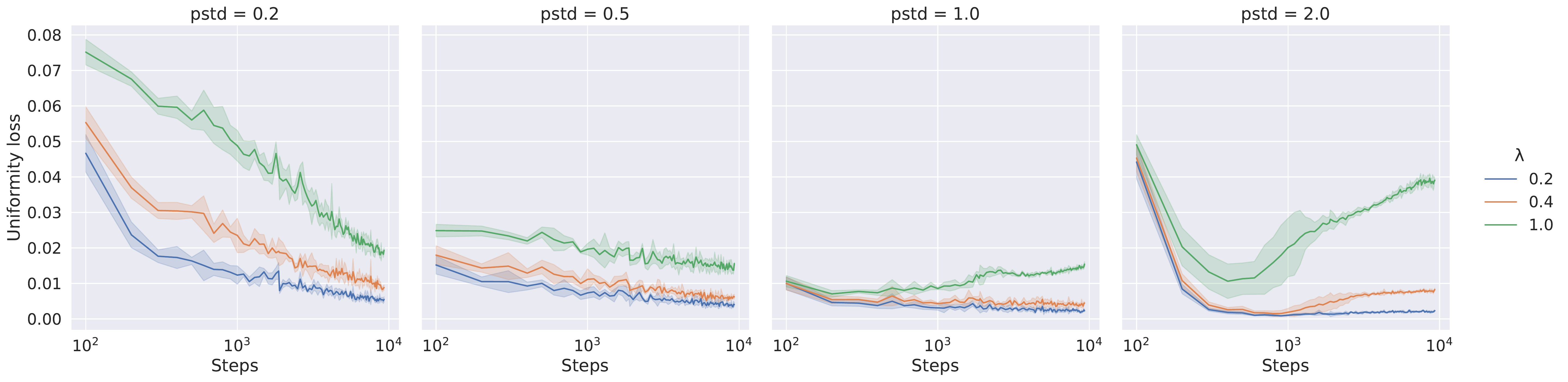}
	\end{center}
	\hspace{1.2cm}(a)~$s=\sqrt{0.2}$\hspace{1.2cm}(b)~$s=\sqrt{0.5}$\hspace{1.3cm}(c)~$s=\sqrt{1.0}$\hspace{1.2cm}(d)~$s=\sqrt{2.0}$
	\caption{We demonstrate $\loss^{uni}(\vtheta_0)$ as the function of optimization steps in log-linear scale, when we pre-train $\fcn$ on MNIST by minimizing \eqref{eq:final_loss} with respect to $\vtheta_0$ while fixing $\xi$ to 1. We plot mean (curve) $\pm$ stdev (shaded area) across four random runs. For $\lambda=1.0$ (green), $\loss^{uni}(\vtheta_0)$ does not sufficiently decrease to 0 for all $s$. Moreover, for $(s,\lambda)\in\{(\sqrt{1.0}, 1.0), (\sqrt{2.0}, 1.0), (\sqrt{2.0}, 1.0)\}$, $\loss^{uni}(\vtheta_0)$ does not monotonically decrease. To make our model predictions uniform over $\Delta^{d-1}$, we need to avoid large $s$ and $\lambda$.}\label{a_fig:uniform_loss}
\end{figure}

Figure \ref{a_fig:uniform_loss} shows that $\loss^{uni}$ does not converge to 0 when $\lambda=1.0$. It means that our model cannot predict uniformly over $\Delta^{d-1}$ when we use $\lambda=1.0$. However, $\loss^{uni}$ is sufficiently closed to 0 for most setups except $\lambda=1.0$. We thus additionally evaluate the quality of $(s,\lambda)$ whether each setup suffers from either the degenerate softmax or the input-output detachment.

For the degenerate softmax, we compute the average ratio (\%) of perturbed models  (\textit{DS}) with the degenerate softmax across 128 mini-batches of size 32, out of 256 perturbed models. In the case of the input-output detachment, we first calculate the average ratio (\%) of fully dead neurons across 128 mini-batches of size 32 at each hidden layer. We then take their maximum ratio (\textit{IOD}). We report the average and standard deviation of DS and IOD across four random experiments.   

\begin{table}[h]
\caption{We present the average ($\pm$stdev) of DS (\%) and IOD (\%) across four random experiments of Figure \ref{a_fig:uniform_loss}. The lower DS means that our model does not have the issue of the degenerate softmax. Similarly, the lower IOD implies that there is less input-output detachment in our model. When we pre-train our model with $s=\sqrt{0.5}$ and $\lambda=0.4$, both DS and IOD are significantly smaller than the other setups. It means that pre-training with $s=\sqrt{0.5}$ and $\lambda=0.4$ prevents both the degenerate softmax and the input-output detachment more effectively than the others.} 
\label{a_tab:s_lambda_choice}
\begin{center}
    \begin{tabular}{cc|ccc}
    \hline
    \multicolumn{1}{c}{$\boldsymbol{s}$}&\multicolumn{1}{c|}{$\boldsymbol{\lambda}$}&\multicolumn{1}{c}{\bf DS (\%)}&\multicolumn{1}{c}{\bf IOD (\%)}
    \\
    \hline 
    \hline
    $\sqrt{0.2}$&0.2&1.32$\pm$0.73&28.87$\pm$5.05\\
    $\sqrt{0.2}$&0.4&0.54$\pm$0.42&30.25$\pm$3.56\\
    $\sqrt{0.2}$&1.0&0.84$\pm$0.53&16.52$\pm$3.27\\
    \hline
    $\sqrt{0.5}$&0.2&1.33$\pm$0.64&17.80$\pm$3.83\\
    $\sqrt{0.5}$&0.4&1.58$\pm$0.72&5.48$\pm$1.56\\
    $\sqrt{0.5}$&1.0&1.41$\pm$0.72&34.15$\pm$28.81\\
    \hline
    $\sqrt{1.0}$&0.2&6.39$\pm$1.19&6.91$\pm$1.19\\
    $\sqrt{1.0}$&0.4&4.30$\pm$1.23&13.06$\pm$3.23\\
    $\sqrt{1.0}$&1.0&3.96$\pm$1.26&16.13$\pm$3.57\\
    \hline
    $\sqrt{2.0}$&0.2&44.20$\pm$4.75&56.44$\pm$5.00\\
    $\sqrt{2.0}$&0.4&35.57$\pm$5.63&47.62$\pm$5.48\\
    $\sqrt{2.0}$&1.0&16.93$\pm$7.63&34.12$\pm$9.02\\
    \hline
\end{tabular}
\end{center}
\end{table}

As shown in Table \ref{a_tab:s_lambda_choice}, although $\loss^{uni}$ is closed to 0 for most $(s, \lambda)$, the degree of the degenerate softmax and the input-output detachment highly depends on both $s$ and $\lambda$. For most $s$ ($s\in\{\sqrt{0.2},\sqrt{0.5},\sqrt{1.0}\}$, our pre-training does not cause the degenerate softmax, but we observe that there is the issue of input-output detachment except $(s,\lambda)\in\{(\sqrt{0.5},0.4),(\sqrt{1.0},0.2)\}$. Since using $s=\sqrt{0.5}$ and $\lambda=0.4$ has lower DS than using $s=\sqrt{1.0}$ and $\lambda=0.2$, we finally choose $s=\sqrt{0.5}$ and $\lambda=0.4$ for all experiments in this paper.

\subsection{Coefficient for preventing degenerate softmax} \label{a_sec:lambda_underclass}
In \S\ref{a_sec:perturbation_std}, we empirically validate that using $s=\sqrt{0.5}$ and $\lambda=0.4$ for \eqref{eq:final_loss} effectively alleviates the issue of both degenerate softmax and input-output detachment with sufficiently small $\loss^{uni}$.  

\section{Experimental Details in \S\ref{sec:main_exp}}
\label{a_sec:main_exp_detail}

\subsection{Model architectures}
\label{a_sec:model_archi}
\paragraph{\fcn} $\fcn$ is a multi-layer perceptron with fully-connected layers. It has 3 hidden layers, and each hidden layer has 392 units activated by ReLU \citep{nair2010rectified}.
\paragraph{\fcn+\bn} We refer as $\fcn+\bn$ to adding batch normalization \citep{ioffe2015batch} for each hidden layer of $\fcn$ before ReLU.

\subsection{Pre-training details}
We note that both our pre-training and random label pre-training \citep{pondenkandath2018leveraging} do not require any label (i.e., unsupervised learning). Both use 60,000 unlabelled examples.
\label{a_sec:pretrain}
\paragraph{Our pre-training} We pre-train $\fcn$ by minimizing \eqref{eq:final_loss} with respect to $\vtheta_0$ with hyperparameters described in \S\ref{sec:final_loss}. We use Adam with a fixed learning rate of $2\times 10^{-4}$, $\beta_1=0.9$, and $\beta_2=0.999$ without any regularization. We run our experiments with a batch size of 32 for 5 epochs. To compute $\loss^{uni}$ and $\loss^{sd}$, we additionally need random samples from both $\normal(\vzero,\mSigma)$ and $\uniform(\Delta^{d-1})$. We draw 256 random samples from each distribution. We use the parameter configuration of which average pre-training loss calculated on each mini-batch across every 100 steps is at its minimum.

\paragraph{Random label pre-training} We pre-train $\fcn$ and $\fcn+\bn$ by random labeling. We assign a label drawn from $\bernoulli(0.5)$ to every instance in each mini-batch. We minimize the cross entropy with these randomly labelled mini-batches of size 32 for 5 epochs. We use Adam with a fixed learning rate of $2\times 10^{-4}$, $\beta_1=0.9$, and $\beta_2=0.999$ without any regularization. We use the parameter configuration of which average pre-training loss computed on each mini-batch across every 100 steps is at its minimum.

\section{Limitations}
In this paper, we did not consider the computational efficiency of our algorithm. There are two factors to make our computational cost expensive. First of all, for $\loss^{uni}$ and $\loss^{sd}$, we need to use a number of perturbed models (256 in our experiments). It means that the computational cost is proportional to the number of perturbed models. Moreover, for $\loss^{iod}$, we used the Jacobian of the model output with respect to either the input vector or the vector of pre-activated neurons at the intermediate layer. Its computational cost is proportional to the dimensionality of the model output.

Although our analysis does not depend on model architectures, the choice of hyperparameters might rely on them. Since we only recommended how to select the hyperparameters of our algorithm for a fully connected network, our proposed algorithm may require additional considerations to apply other networks.
\end{document}

%% file: math_commands.tex
\usepackage{amsmath,amsfonts,bm}

\def\eqref#1{equation~\ref{#1}}
\def\Eqref#1{Equation~\ref{#1}}

\def\1{\bm{1}}

\def\rx{{\textnormal{x}}}
\def\ry{{\textnormal{y}}}
\def\rz{{\textnormal{z}}}

\def\vzero{{\bm{0}}}
\def\vone{{\bm{1}}}

\def\vtheta{{\bm{\theta}}}
\def\va{{\bm{a}}}
\def\vb{{\bm{b}}}
\def\vc{{\bm{c}}}

\def\ve{{\bm{e}}}
\def\vf{{\bm{f}}}

\def\vs{{\bm{s}}}

\def\vu{{\bm{u}}}
\def\vv{{\bm{v}}}

\def\vx{{\bm{x}}}
\def\vy{{\bm{y}}}

\def\mA{{\bm{A}}}

\def\mJ{{\bm{J}}}

\def\mW{{\bm{W}}}

\def\mSigma{{\bm{\Sigma}}}

\DeclareMathAlphabet{\mathsfit}{\encodingdefault}{\sfdefault}{m}{sl}
\SetMathAlphabet{\mathsfit}{bold}{\encodingdefault}{\sfdefault}{bx}{n}

\def\sA{{\mathbb{A}}}
\def\sB{{\mathbb{B}}}

\def\sF{{\mathbb{F}}}

\def\sP{{\mathbb{P}}}

\def\sV{{\mathbb{V}}}
\def\sW{{\mathbb{W}}}

\newcommand{\E}{\mathbb{E}}

\newcommand{\R}{\mathbb{R}}

\newcommand{\normltwo}{L^2}

\DeclareMathOperator*{\argmin}{arg\,min}